\pdfoutput=1
\documentclass{article}

\ifdefined\pdfsuppressptexinfo
\fi

\PassOptionsToPackage{numbers,sort&compress}{natbib}
\usepackage[final]{neurips_2026}
\makeatletter
\providecommand{\@trackname}{\@neuripsordinal\ Conference on Neural Information Processing Systems (NeurIPS \@neuripsyear).}
\makeatother

\usepackage{placeins}
\usepackage[utf8]{inputenc}
\usepackage[T1]{fontenc}
\usepackage{hyperref}
\usepackage{url}
\usepackage{wrapfig}
\usepackage{microtype}
\usepackage{graphicx}
\usepackage{subcaption}
\usepackage{booktabs}
\usepackage{nicefrac}
\usepackage{amsmath}
\usepackage{amsfonts}
\usepackage{amssymb}
\usepackage{mathtools}
\usepackage{amsthm}
\usepackage{xcolor}
\usepackage{multirow}
\usepackage{makecell}
\usepackage[most]{tcolorbox}
\tcbuselibrary{breakable}
\usepackage{algorithm}
\usepackage{algorithmic}
\usepackage{enumitem}
\usepackage[capitalize,noabbrev]{cleveref}
\usepackage[textsize=tiny]{todonotes}
\usepackage[table]{xcolor}
\newcommand{\best}[1]{\cellcolor{blue!10}\textbf{#1}}
\FloatBarrier

\setcitestyle{numbers,square,comma,sort&compress}
\let\cite\citep

\definecolor{MyPurple}{HTML}{67849B}
\definecolor{MyLightBlue}{HTML}{F4F7FC}
\tcbset{
  myexample/.style={
    colback=white,
    colframe=MyPurple,
    colbacktitle=MyLightBlue,
    coltitle=black!80,
    fonttitle=\bfseries,
    arc=4mm,
    boxrule=1pt,
    left=3mm, right=3mm, top=1.5mm, bottom=1.5mm,
    title={#1}
  }
}

\definecolor{caseblue}{HTML}{4E79A7}
\tcbset{
  casestudy/.style={
    colback=white,
    colframe=black!60,
    colbacktitle=caseblue,
    coltitle=white,
    fonttitle=\bfseries,
    arc=4mm,
    boxrule=1pt,
    left=3mm, right=3mm, top=1.5mm, bottom=1.5mm,
    breakable,
    title={#1}
  }
}

\theoremstyle{plain}
\newtheorem{theorem}{Theorem}[section]

\newtheorem{lemma}[theorem]{Lemma}

\theoremstyle{definition}

\newtheorem{assumption}[theorem]{Assumption}
\theoremstyle{remark}

\definecolor{cGreen}{RGB}{0,150,0}

\definecolor{brown}{RGB}{139,64,0}

\definecolor{easy-top}{RGB}{160,208,208}
\definecolor{difficult-top}{RGB}{213,170,190}

\hypersetup{
  hidelinks,
  pdftitle={Train at the Moving Edge: Efficient RL for Large Reasoning Models via Rollout Selection},
  pdfauthor={Jiahao Wu, Ning Lu, Shengcai Liu, Kun Wang, Yanting Yang, Baijiong Lin, Chen Jason Zhang, Qing Li, Ke Tang},
  pdfkeywords={Machine Learning, NeurIPS, Reinforcement Learning, Large Reasoning Models, Prompt Selection}
}

\title{Train at the Moving Edge: Efficient RL for Large Reasoning Models via Rollout Selection}

\author{%
\begin{minipage}[t]{\dimexpr\textwidth-2\tabcolsep-2pt\relax}
\centering
\normalfont\normalsize
{\bfseries
Jiahao Wu\textsuperscript{1,2,*} \quad
Ning Lu\textsuperscript{1,3,*} \quad
Shengcai Liu\textsuperscript{1,\textdagger} \quad
Kun Wang\textsuperscript{4} \quad
Yanting Yang\textsuperscript{5}
}\\[0.1em]%
{\bfseries
Baijiong Lin\textsuperscript{6} \quad
Chen Jason Zhang\textsuperscript{2} \quad
Qing Li\textsuperscript{2} \quad
Ke Tang\textsuperscript{1}
}\\[0.4em]
{\normalsize
\textsuperscript{1}Guangdong Provincial Key Laboratory of Brain-Inspired Intelligent Computation,\\[0.1em]
Department of CSE, SUSTech\quad \textsuperscript{2}Department of Computing, PolyU\\[0.1em]%
\textsuperscript{3}Department of CSE, HKUST\quad\textsuperscript{4}College of Computing and Data Science, NTU
}\\[0.1em]%
{\normalsize
\textsuperscript{5}Department of CS, Rutgers \quad
\textsuperscript{6}AI Thrust, HKUST(GZ)
}
\end{minipage}
}

\begin{document}

\maketitle
\begingroup
\renewcommand{\thefootnote}{\fnsymbol{footnote}}
\footnotetext[1]{Equal contribution. \quad
\textsuperscript{\textdagger}Corresponding author. \quad
\hypertarget{hive-code-link}{\textsuperscript{\ensuremath{\ddagger}}Code:}
\href{https://github.com/JiahaoWuGit/HIVE-Efficient-Rollout-RL}{JiahaoWuGit/HIVE-Efficient-Rollout-RL}.}
\endgroup
\suppressfloats[t]

\begin{abstract}
Reinforcement learning (RL) has become essential for post-training large language models (LLMs) in reasoning tasks. While scaling rollouts can stabilize training and enhance performance, it introduces substantial computational overhead. In algorithms like GRPO, multiple rollouts per prompt incur prohibitive costs, as a large portion of prompts provide negligible gradients and are thus of low utility. This raises a key question: \textit{how to identify high-utility prompts before an expensive rollout?}
Our experimental analysis reveals that sample utility is non-uniform and dynamic: the strongest learning signals concentrate at the ``learning edge'', the intersection of intermediate difficulty and high response entropy, which shifts throughout training. Motivated by this observation, we propose HIVE, a history-informed and online-verified prompt selection framework for data-efficient RL training. HIVE first uses historical reward statistics and {response entropy} as a cheap prior to filter candidate prompts, and then employs prompt entropy as a real-time proxy to prune instances with stale utility. Across multiple reasoning benchmarks and base models, HIVE maintains  reasoning accuracy with substantial rollout cost reduction, achieving up to 2.3× total training speedup.
Code is here\hyperlink{hive-code-link}{\textsuperscript{\ensuremath{\ddagger}}}.

\end{abstract}

\begin{figure}[t]
    \centering
    \vskip -0.15in
    \includegraphics[width=0.95\linewidth]{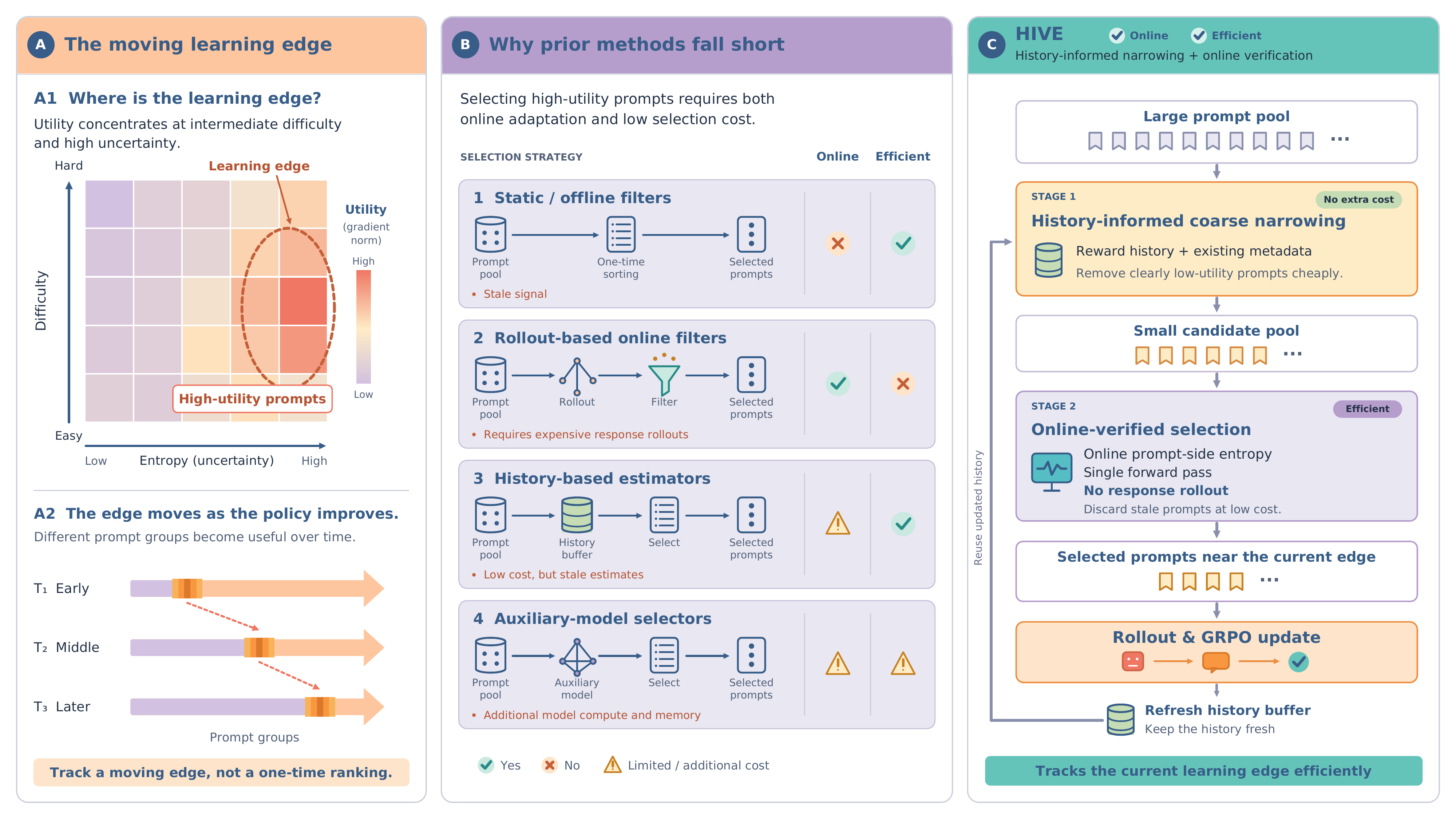}
    \vskip -0.05in
    \caption{\textbf{RLVR prompt selection is a moving-edge tracking problem.} (A) Utility concentrates at
    intermediate difficulty and high uncertainty, and the learning edge is dynamic.
    (B) Prior methods face limitations in efficiently selecting online high-utility prompts. (C) HIVE combines history-informed coarse narrowing with low-cost online
    verification to track the current learning edge.}
    \label{fig:intro-compare-motivate}
    \vskip -0.2in
\end{figure}

\section{Introduction}
\label{sec:introduction}
Reinforcement Learning (RL) has emerged as a prevalent paradigm for fine-tuning large language models (LLMs), particularly for enhancing complex mathematical and logical reasoning capabilities~\cite{guo2025deepseek,yu2025dapo,tomihari2026learningdynamicsrlposttraining}. While advanced RL algorithms such as group relative policy optimization (GRPO)~\cite{shao2024deepseekmath} rely on rollout scaling to stabilize training and enhance model performance, they introduce a severe computational overhead~\cite{xu2025not,yu2025dapo,li2026stalefeedbackcoevolvingcritics}. In the training of those RL algorithms, a large portion of computational resources is wasted on generating rollouts for low-utility prompts that are either too trivial or too intractable for the current model. This results in vanishing gradients and wasted resources, providing no informative learning signal for model updates~\cite{zheng2025act,yu2025dapo,noukhovitch2025fasterICLR,verl2025}. Therefore, we investigate the following research question in this paper: \textit{how to identify prompts that are likely to produce useful gradients before rollout?}

Existing methods face limitations in efficiently selecting {high-utility} prompts in an online manner. Static or offline filters rank prompts before training using fixed heuristics, making them inexpensive but unable to adapt to the changing model during RL
training~\cite{Wang2025ReinforcementLF,Li2025LIMRLI,ye2025limo}. Rollout-based online filters, such as dynamic sampling, better capture the current utility of each prompt by checking reward variance from newly generated responses~\cite{yu2025dapo,xu2025not,Chen2025LSPO,Ye2025PROF}. However, they already incur substantial rollout cost. History-based estimators such as GRESO~\cite{zheng2025act} avoid additional
rollout cost by estimating prompt utility from past reward statistics, but the metadata can become stale as the policy
updates~\cite{zheng2025act,sun2025dots,li2026stalefeedbackcoevolvingcritics}. Auxiliary-model selectors, such as PCL~\cite{Gao2025PCL}, attach external or jointly trained predictors to estimate prompt difficulty or utility, requiring extra compute and memory.

To identify what makes a prompt useful during RL training, we first empirically analyze GRPO training dynamics. Among prompts with non-zero gradients, those with high response-side entropy produce greater learning progress under the same {rollout budget} (Figure~\ref{fig:intro-high-entropy}). More broadly, utility, measured by the length-normalized gradient norm, concentrates at the intersection of intermediate difficulty and high response-side entropy (Figure~\ref{fig:intro-learning-edge}), which we refer to as the learning edge. However, this edge is dynamic: {prompts selected using historical metadata quickly lose real-time utility, whereas prompts verified under the current policy are more likely to be effective (Figure~\ref{fig:intro-history-staleness})}; prompts move substantially among High, Low, and Zero utility classes across training steps (Figure~\ref{fig:intro-utility-shift}). These observations suggest that an effective prompt selection framework for efficient LLM RL should be \textbf{online}, to track the moving learning edge, and \textbf{efficient}, to avoid excessive selection overhead.

Based on previous analysis, we propose \textbf{HIVE} (\textbf{H}istory-\textbf{I}nformed \textbf{V}erification of the learning \textbf{E}dge), a dual-stage
framework for rollout-efficient RLVR (Figure~\ref{fig:framework}). HIVE first performs history-informed coarse narrowing using
reward trajectories and {response-side entropies} accumulated from previous training steps. This stage targets efficiency by retaining exploration probability for prompts that are around the edge. HIVE then utilizes prompt entropy as a {high-fidelity and efficient proxy} to perform online verification under the current policy, filtering stale candidates before the rollout phase.

Evaluation for HIVE is conducted across multiple reasoning benchmarks (e.g., math and tool calling) and models: Qwen3-4B-Base~\cite{yang2025qwen3technicalreport},
Qwen2.5-Math-1.5B/7B~\cite{yang2024qwen25mathtechnicalreportmathematical}, DeepSeek-R1-Distill-Qwen-1.5B~\cite{guo2025deepseek},
Qwen2.5-14B/32B~\cite{qwen2.5-large}, and Llama-3.2-3B-Instruct~\cite{grattafiori2024llama3}. {We compare the observed accuracy--cost tradeoff of HIVE with} baselines like Dynamic Sampling~\cite{yu2025dapo}, GRESO~\cite{zheng2025act} and PCL~\cite{Gao2025PCL}, achieving up to \textbf{2.3$\times$ faster total training time} for models like Qwen2.5-Math-7B. Finally, we show that HIVE lowers computational
overhead, achieving \textbf{up to 9.2 million rollouts reduction} {with broadly comparable reported reasoning scores.} Beyond mathematical reasoning, HIVE also attains {a similar observed tool-calling score on our BFCL-V2 split} while using \textbf{4.8$\times$ fewer rollouts} than Dynamic Sampling.%

\section{Empirical Observations: The Moving Learning Edge}
\label{sec:obs-preliminary}
This section dissects the training dynamics of GRPO-style RLVR to identify what makes a
prompt useful for policy optimization and why effective selection in RLVR is an online tracking problem.
We first recall why low-utility prompts produce no GRPO learning signal, then show that the useful
prompts concentrate at a moving edge characterized by intermediate difficulty and high uncertainty. {The experimental details are reported in Appendix~\ref{sec:preliminary-setup}.}

\begin{figure}[t]
    \centering
    \captionsetup[subfigure]{labelformat=parens,labelsep=none,font=footnotesize,justification=centering}
    \begin{subfigure}[t]{0.241\linewidth}
        \centering
        \includegraphics[width=\linewidth]{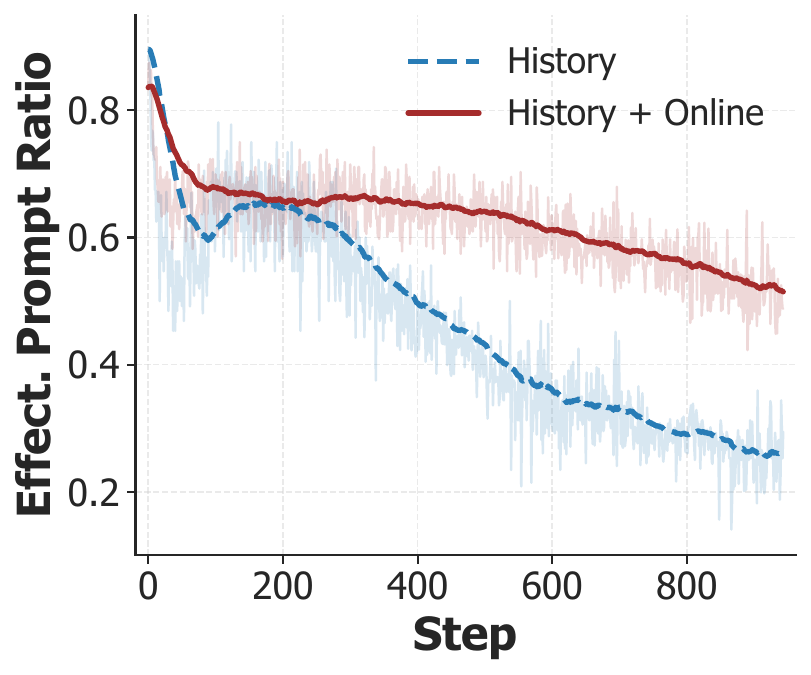}
        \vspace{-6pt}
        \caption{}
        \label{fig:intro-history-staleness}
    \end{subfigure}\hfill
    \begin{subfigure}[t]{0.245\linewidth}
        \centering
        \includegraphics[width=\linewidth]{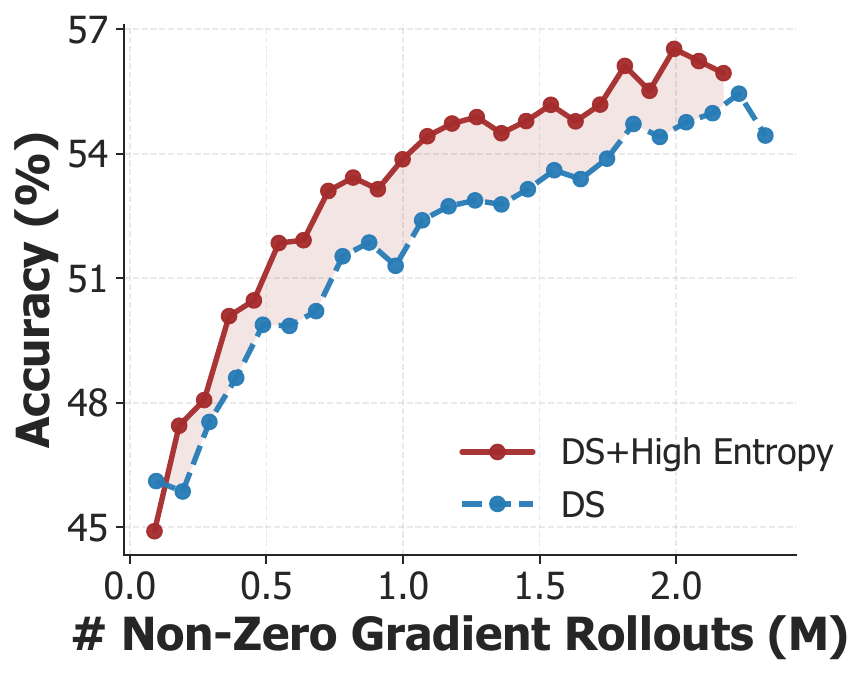}
        \vspace{-6pt}
        \caption{}
        \label{fig:intro-high-entropy}
    \end{subfigure}\hfill
    \begin{subfigure}[t]{0.245\linewidth}
        \centering
        \includegraphics[width=\linewidth]{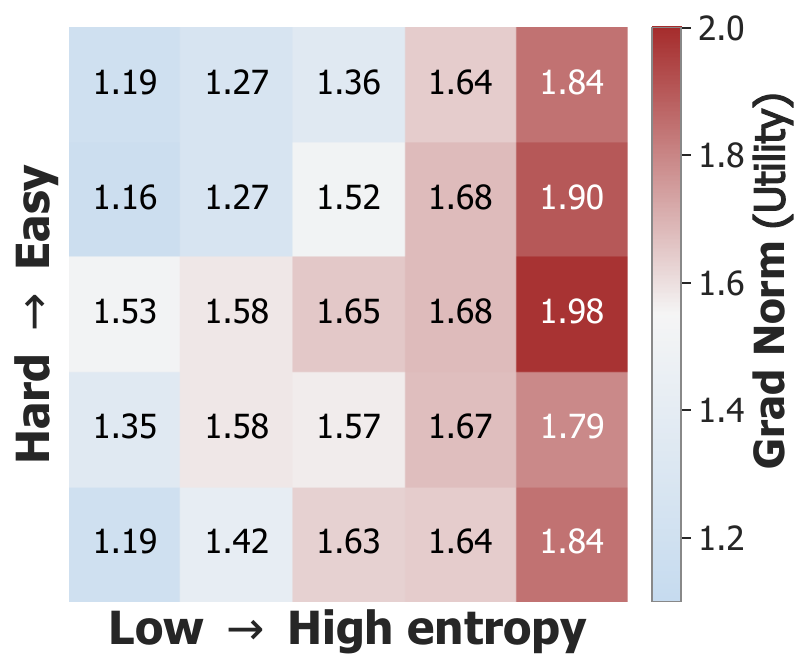}
        \vspace{-6pt}
        \caption{}
        \label{fig:intro-learning-edge}
    \end{subfigure}\hfill
    \begin{subfigure}[t]{0.241\linewidth}
        \centering
        \includegraphics[width=\linewidth]{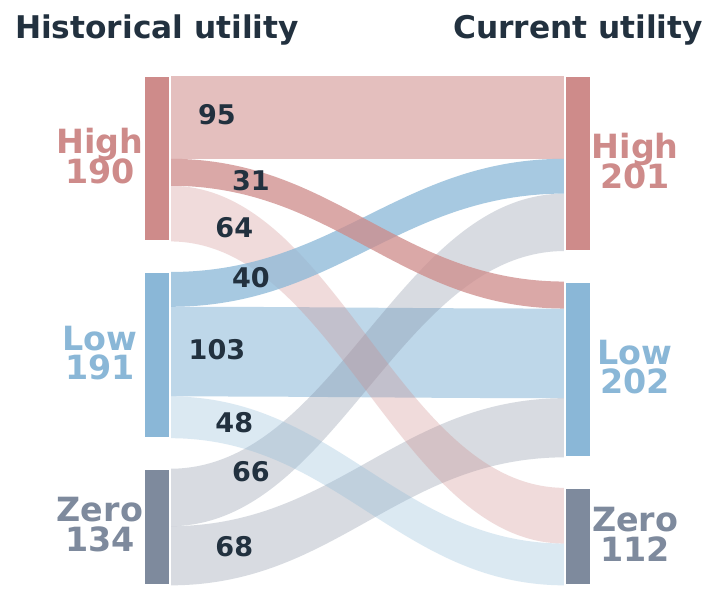}
        \vspace{-6pt}
        \caption{}
        \label{fig:intro-utility-shift}
    \end{subfigure}
    \vskip -0.1in
    \caption{\textbf{Empirical observations.}
    (a) Selection only based on historical statistics becomes increasingly stale during training, while online verification retains more effective prompts.
    (b) Among prompts with non-zero gradients, prioritizing high-entropy
        prompts improves accuracy at the same rollout budget.
    (c) Gradient utility concentrates at an intermediate-difficulty,
        high-entropy learning edge.
    {(d) A large fraction of sampled prompts shift between High, Low, and Zero utility classes across steps 500 and 1000, indicating that prompts selected by historical metadata no longer match the current policy’s utility distribution.}
    }
    \label{fig:intro-a-b}
    \vskip -0.2in
\end{figure}

\begin{figure}[t]
    \centering
    \includegraphics[width=1.0\linewidth]{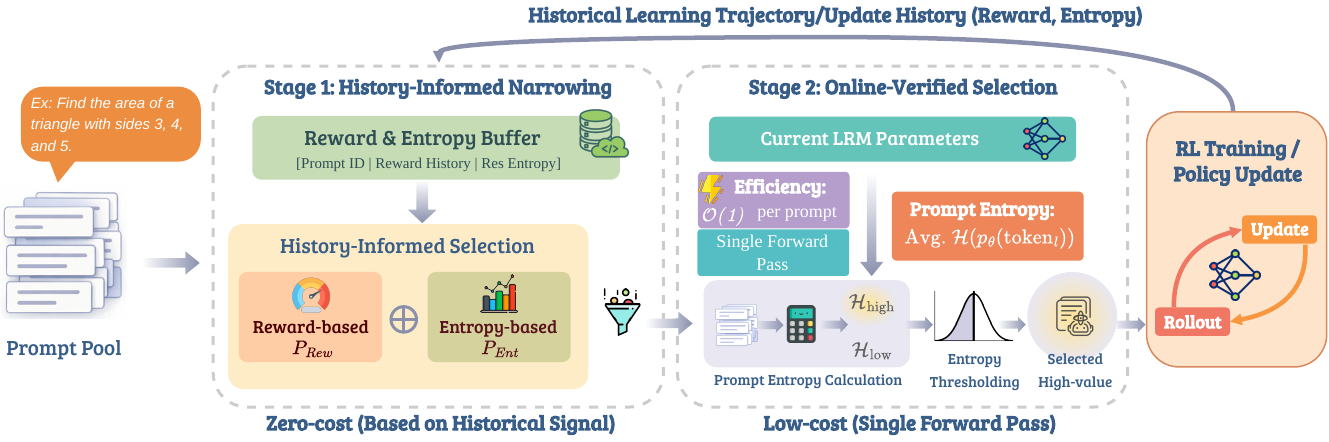}
    \caption{Overview of HIVE. \textbf{HIVE tracks the moving learning edge by decomposing prompt selection into two stages.}
    \textbf{Stage 1} uses historical reward trajectories and response-side entropy for efficient coarse narrowing. \textbf{Stage 2} computes prompt-side entropy under the current policy for low-cost online verification, rejecting stale candidates before rollout.}
    \label{fig:framework}
    \vskip -0.15in
\end{figure}

\textbf{Zero-variance prompts produce vanishing advantages.}
Group relative policy optimization (GRPO)~\cite{shao2024deepseekmath} is widely used in RLVR for reasoning models~\cite{guo2025deepseek,yang2025qwen3technicalreport}. For a prompt $q$, GRPO samples a group of responses $\{o_i\}_{i=1}^G$ and optimizes
\begin{equation}
\label{eq:grpo}
\begin{aligned}
\mathcal{J}_{\mathrm{GRPO}}(\theta)
&=
\mathbb{E}_{\substack{
q \sim P(Q),\ \{o_i\}_{i=1}^{G} \sim \pi_{\theta_{\mathrm{old}}}(\cdot \mid q)
}}
\Biggl\{
\frac{1}{G}\sum_{i=1}^{G}\frac{1}{|o_i|}
\sum_{t=1}^{|o_i|}
\biggl(
\min \Bigl[
\frac{\pi_{\theta}(o_{i,t}\mid q,o_{i,<t})}{\pi_{\theta_{\mathrm{old}}}(o_{i,t}\mid q,o_{i,<t})}\hat{A}_{i,t}, \\
&\qquad\qquad\;
\operatorname{clip}(
\frac{\pi_{\theta}(o_{i,t}\mid q,o_{i,<t})}{\pi_{\theta_{\mathrm{old}}}(o_{i,t}\mid q,o_{i,<t})}, 1-\varepsilon, 1+\varepsilon
)\hat{A}_{i,t}
\Bigr]
-\beta D_{\mathrm{KL}}\!\left[
\pi_{\theta}\,\|\,\pi_{\mathrm{ref}}
\right]
\biggr)
\Biggr\},
\end{aligned}
\end{equation}
where $P(Q)$ is the prompt distribution, $G$ is the response group size, and $o_{i,t}$ denotes the $t$-th token
of response $o_i$. The KL penalty with weight $\beta$ anchors $\pi_\theta$ to the
reference policy $\pi_{\mathrm{ref}}$. Since RLVR gives a sequence-level reward $r_i$, all tokens in $o_i$
share the group-normalized advantage $\hat{A}_{i,t}=\hat{A}_{i}$:
\begin{equation}
\label{eq:grpo_advantage}
\hat A_i=\begin{cases}
(r_i-\bar r)/\sigma_r,&\sigma_r>0,\\
0,&\sigma_r=0,
\end{cases}
\quad \bar r=\frac1G\sum_{j=1}^G r_j,\quad
\sigma_r=\mathrm{std}(\{r_j\}_{j=1}^G).
\end{equation}
When all sampled responses for a prompt receive the same reward, the reward variance is zero and
the prompt contributes no informative policy-gradient signal. Such prompts are either already easy
for the current policy or still too hard to discriminate among sampled responses. Rollout scaling
therefore wastes compute unless it avoids these zero-variance regions.

\textbf{Utility concentrates at an intermediate-difficulty, high-entropy edge.}
Zero-variance filtering is necessary but not sufficient: among prompts with non-zero gradients, some still provide much stronger learning signal than others. Following previous analyses~\cite{kim2025mitigatinglengthbiasrlhf,paul2021DataDiet,fayyaz2022bertdatadietfinding,singhal2024aCOLM}, we use length-normalized gradient norm as a proxy for per-prompt utility and study how it varies with empirical difficulty and response-side entropy.

Figure~\ref{fig:intro-learning-edge} shows two consistent patterns. First, within a fixed difficulty band, gradient norm increases with response-side entropy. Second, at a fixed entropy level, utility follows an inverted-U shape over
difficulty: trivial prompts and intractable prompts have low utility, while intermediate-difficulty prompts have
higher utility. The strongest learning signal appears at the intersection of these two conditions. We call this
region the \textit{learning edge}. This observation explains why merely removing zero-variance prompts is too
coarse: \textit{the best rollout budget should be allocated to prompts that are both learnable and uncertain under the
current policy}.

\textbf{Historical metadata becomes stale as the edge moves.}
As the policy updates, the model's state of knowledge changes, and historical
metadata such as reward trajectories and response entropies can quickly become stale~\cite{cui2025entropymechanismreinforcementlearning,tomihari2026learningdynamicsrlposttraining,li2026stalefeedbackcoevolvingcritics}. To quantify this effect, we compare prompts selected by historical metadata from the previous training window
with prompts selected by online metrics recomputed under the current policy. As shown in Figure~\ref{fig:intro-learning-edge},
historical selection yields lower current gradient norms, whereas online selection isolates prompts with high utility. Figure~\ref{fig:intro-utility-shift} further shows that prompts frequently move among High, Low, and Zero utility classes across training steps. Together, these results indicate that: \textit{a selector must revalidate candidate prompts under the current policy to keep tracking the moving edge}.

\section{Methodology}
\label{sec:method}
\begingroup
\setlength{\abovedisplayskip}{5pt plus 1pt minus 1pt}
\setlength{\belowdisplayskip}{5pt plus 1pt minus 1pt}
\setlength{\abovedisplayshortskip}{3pt plus 1pt minus 1pt}
\setlength{\belowdisplayshortskip}{3pt plus 1pt minus 1pt}
Section~\ref{sec:obs-preliminary} shows that useful RLVR prompts concentrate near a moving learning edge and historical metadata alone is unreliable because the edge moves as the policy changes. Based on these observations, HIVE is designed as a two-stage pre-rollout selector (Figure~\ref{fig:framework} and Algorithm~\ref{alg:hive-training}). Stage~1 uses historical reward and {response-entropy traces} as a cheap prior to form a candidate pool; Stage~2 revalidates those candidates under the current policy using prompt-side entropy before expensive rollouts for the standard GRPO update.

\subsection{Design Objective}
A practical RLVR prompt selector must be both \textbf{online}, adapting as the policy evolves and the learning frontier shifts during training, and \textbf{efficient}, avoiding full response rollouts that could erase its computational benefits. HIVE achieves this via two stages. Stage~1 targets efficiency: it uses historical reward trajectories and {response-side entropies} to cheaply narrow the full prompt pool to a candidate set. Stage~2 restores online precision: it verifies candidates with prompt-side entropy computed under the current model parameters. The final selected prompts are then passed to the standard rollout and GRPO update. The detailed procedure is provided in Appendix~\ref{appendix-sec:alg}.

\subsection{Stage 1: History-Informed Coarse Narrowing}
Stage~1 removes prompts that are unlikely to lie near the learning edge while keeping enough exploration to handle edge movement. {It combines a reward-based score, which tracks difficulty saturation, with an entropy-based score, which favors historically uncertain prompts.}

\textbf{Reward-history score.} For each prompt $x_i\in\mathcal{D}$, HIVE stores a trace $\mathcal{T}_i^{(t)}=\{R_{i,j}\}_{j=1}^{n_i}$ before training step $t$, where $R_{i,j}=\{r_{i,j}^{(k)}\}_{k=1}^{G}$ is the reward group recorded at the $j$-th previous visit. Let $\zeta_{i,j}=\mathbf{1}\{\operatorname{Var}(R_{i,j})=0\}$ indicate whether that visit produced a zero-variance reward group. HIVE summarizes recent saturation by the length of the trailing zero-variance run:
\begin{equation}
\label{eq:stage1_reward_score}%
\begin{aligned}
&z_i^{(t)}
=\max\left\{k\in\{0,\ldots,n_i\}:
\prod_{j=n_i-k+1}^{n_i}\zeta_{i,j}=1\right\},\\
&s_{\mathrm{rew}}(x_i)
=\max\!\Bigl\{\left(p_{e,\tau_i}^{(t)}\right)^{z_i^{(t)}},\epsilon_p\Bigr\},
\qquad \tau_i\in\{\mathrm{easy},\mathrm{hard}\}.
\end{aligned}
\end{equation}
Here $\tau_i$ distinguishes all-correct and all-incorrect zero-variance prompts, $p_{e,\tau_i}^{(t)}\in(0,1)$ is the corresponding exploration probability, and $\epsilon_p=0.01$ is a per-prompt score floor. Repeated zero-variance outcomes decay the score exponentially down to this floor, preserving exploration for prompts that may later become learnable as the policy improves.

HIVE adapts $p_{e,\tau}$ online using the observed zero-variance ratio $\hat{\rho}_{\tau}^{(t)}$: the probability decreases above target $\alpha$, increases below it, and remains unchanged at equality:
\begin{equation}
\label{eq:adaptive_pe}%
p_{e,\tau}^{(t+1)}
=
\left[
p_{e,\tau}^{(t)}
+\Delta p\cdot\operatorname{sign}\!\left(\alpha-\hat{\rho}_{\tau}^{(t)}\right)
\right]_{[p_{\min},p_{\max}]},
\qquad \tau\in\{\mathrm{easy},\mathrm{hard}\},
\end{equation}
where $[\cdot]_{[p_{\min},p_{\max}]}$ denotes projection to a valid probability range. This update prevents HIVE from persistently over-sampling either mastered prompts or currently intractable prompts.

\textbf{Response-entropy score.}
Historical response-side entropy provides a low-cost prior for uncertainty. For the most recent rollout group of $x_i$, let
\begin{equation}
\label{eq:stage1_entropy_score}
\begin{aligned}
&U_i^{(r)}
=\frac{1}{L_{i,r}}\sum_{\ell=1}^{L_{i,r}}
\mathcal{H}\!\left(p_{\theta}(\cdot\mid x_i,o_{i,<\ell}^{(r)})\right),
\qquad
H_i=\frac{1}{G}\sum_{r=1}^{G}U_i^{(r)},\\
&s_{\mathrm{ent}}(x_i)
=\frac{H_i-H_{\min}}{H_{\max}-H_{\min}+\epsilon},
\qquad\;\;
\mathcal{H}(p)=-\sum_{v\in\mathcal{V}}p_v\log p_v .
\end{aligned}
\end{equation}
The normalization is computed within the current metadata pool. It makes the entropy score comparable across training stages even as the policy becomes more confident.

\textbf{Candidate sampling.}
The Stage~1 acceptance probability combines recent reward saturation and historical uncertainty:
\begin{equation}
\label{eq:stage1_select}%
P_{\mathrm{S1}}(x_i)
=
\lambda s_{\mathrm{rew}}(x_i)
+(1-\lambda)s_{\mathrm{ent}}(x_i),
\qquad \lambda\in[0,1].
\end{equation}
HIVE samples prompts with Bernoulli probability $P_{\mathrm{S1}}(x_i)$, yielding a candidate set $\mathcal{C}_t$ for Stage~2. Stage~1 reuses existing metadata and requires no extra response generation or model forward pass.%

\subsection{Stage 2: Online-Verified Selection}
\label{subsec:stage2}
Stage~1 is efficient, but Section~\ref{sec:obs-preliminary} shows why it cannot be the final decision rule: reward and entropy traces become stale as the current policy moves. Stage~2 therefore verifies the candidate set $\mathcal{C}_t$ under the current policy $\pi_{\theta_t}$ before rollouts are generated. Instead of generating $G$ responses per candidate, HIVE uses prompt-side entropy as a lightweight proxy for current uncertainty.

\textbf{Prompt-side entropy.}
For a prompt $x=(x_1,\ldots,x_{L_x})$, HIVE computes
\begin{equation}
\label{eq:prompt_entropy}
V_t(x)
=
\frac{1}{L_x-1}\sum_{\ell=2}^{L_x}
\mathcal{H}\!\left(p_{\theta_t}(\cdot\mid x_{<\ell})\right).
\end{equation}
This score is online because it is evaluated with the current policy parameters, and efficient because it only requires a prompt-side forward pass. In the HIVE pipeline, this verifier complements Stage~1: reward history filters repeated zero-variance regions, while $V_t(x)$ checks whether a historically promising candidate still exhibits current-policy uncertainty.

We give a conditional ordering result for prompt and response entropies. For a fixed checkpoint and prompt set, assume the relation between response entropies and prompt entropies: $U^*(x)=aV_t(x)+b+e(x)$ with $a>0$ and $|e(x)|\leq\varepsilon_{\mathrm{cal}}$ (Appendix~\ref{subsec:thm-theoretical-proof}). %
\begin{theorem}[Conditional rank consistency]
\label{thm:ranking-informal}
Fix two prompts before drawing $n$ independent diagnostic rollouts per prompt. Under the stated calibration assumption, for any $\delta\in(0,1)$, their empirical response-entropy ordering agrees with their prompt-entropy ordering with probability at least $1-\delta$ if
\begin{equation}
a|V_t(x)-V_t(x')|>2\varepsilon_{\mathrm{cal}}+
2\log|\mathcal V|\sqrt{\frac{\log(4/\delta)}{2n}}.
\end{equation}
\end{theorem}
The theorem provides conditional pairwise entropy ordering, explaining the efficacy that prompt-side entropy can serve as a low-cost online proxy. The proof is given in Appendix~\ref{appendix-sec:proof-thm}. The binned associations (Appendix~\ref{appendix-subsec:entropy-robustness}) and stage ablations (Section~\ref{subsec:exp-component-study}) also motivate this proxy empirically.%

\textbf{Median verification gate.}
Given the Stage~1 candidate set $\mathcal{C}_t$, HIVE uses a median gate:
\begin{equation}
\label{eq:median_gate}%
\gamma_t=\operatorname{median}\{V_t(x):x\in\mathcal{C}_t\},
\qquad
\mathcal{B}_t=\{x\in\mathcal{C}_t:V_t(x)\ge\gamma_t\}.
\end{equation}
The final set $\mathcal{B}_t$ is the rollout batch for the current GRPO step. The median threshold keeps throughput stable, adapts to the current candidate distribution, and rejects stale low-uncertainty prompts.

\textbf{Computational cost.}
For $|\mathcal{C}_t|$ candidates, Stage~2 costs $\mathcal{O}(|\mathcal{C}_t|)$ times forward computation. A rollout-based verifier would require $\mathcal{O}(|\mathcal{C}_t|GL_r)$ times. Thus prompt-side verification is negligible relative to multi-rollout filtering. Section~\ref{subsec:further-experiment} empirically confirms this: online verification adds only 0.82 seconds per iteration, less than 0.4\% of the step time.
\endgroup

\begin{table*}[t]
    \centering
    \caption{Comprehensive evaluation of HIVE on math reasoning benchmarks and efficiency analysis. We train multiple models on DAPO+MATH. (a) HIVE maintains comparable accuracy while reducing the number of rollouts. (b) HIVE lowers rollout cost, achieving up to 3.8$\times$ rollout speedup and 2.3$\times$ total training speedup over Dynamic Sampling.}
    \label{tab:main_comparison}

    \begin{subtable}[t]{0.666\textwidth}
        \centering
        \caption{Performance (\%) comparison.} %
        \label{tab:model_performance}
        \resizebox{\linewidth}{!}{
            \begin{tabular}{c|cccccc|c|c}
            \toprule
            {Method} & {Math500} & {AIME24} & {AMC} & {Gaokao} & {Miner.} & {Olymp.} & {Avg.} & {\#Rollout} \\
            \midrule

            \multicolumn{9}{c}{{\textit{Qwen2.5-Math-1.5B}}} \\
            \midrule
            DS    & 77.3 & 16.7 & 61.7 & 64.2 & 31.8 & 38.7 & 48.4 & 7.6M (1.0$\times$) \\
            Pre-filter & 76.8 & 20.8 & 50.0 & 62.1 & 33.2 & 38.5 & {46.9} & 3.3M (2.3$\times$) \\
            GRESO & 77.3 & 15.0 & 59.3 & 66.2 & 32.6 & 38.5 & 48.1 & 3.3M (2.3$\times$)\\
            PCL   & 74.5 & 21.7 & 47.9 & 61.7 & 33.6 & 37.5 & {46.2} & 3.3M (2.3$\times$) \\
            HIVE  & 77.9 & 16.7 & 60.2 & 66.1 & 31.5 & 40.3 & 48.8 & \best{3.1M (2.5$\times$)} \\
            \midrule

            \multicolumn{9}{c}{{\textit{DeepSeek-R1-Distill-Qwen-1.5B}}} \\
            \midrule
            DS    & 87.9 & 36.7 & 71.7 & 78.7 & 35.3 & 54.9 & 60.9 & 2.4M (1.0$\times$)\\
            Pre-filter & 79.7 & 22.5 & 55.1 & 70.5 & 27.2 & 47.5 & {50.4} & 1.6M (1.5$\times$) \\
            GRESO & 85.5 & 37.5 & 70.7 & 76.2 & 34.3 & 52.1 & 59.4 & 1.6M (1.5$\times$)\\
            PCL   & 82.4 & 26.7 & 60.5 & 71.7 & 29.9 & 50.3 & {53.6} & 1.6M (1.5$\times$)\\
            HIVE  & 87.2 & 37.5 & 70.1 & 77.2 & 35.1 & 53.2 & 60.1 & \best{1.5M (1.6$\times$)} \\
            \midrule

            \multicolumn{9}{c}{{\textit{Llama3.2-3b-Instruct}}} \\
            \midrule
            DS    & 52.4 & 16.7 & 46.0 & 49.2 & 20.0 & 20.2 & 34.1 & {11.4M (1.0$\times$)} \\
            Pre-filter & 48.1 & 4.2 & 20.2 & 39.8 & 17.8 & 16.8 & 24.5 & 5.9M (1.9$\times$) \\
            GRESO & 51.8 & 16.7 & 46.3 & 49.4 & 19.7 & 19.6 & 33.9 & 5.9M (1.9$\times$) \\
            PCL   & 53.1 & 15.8 & 27.1 & 48.2 & 22.5 & 19.3 & {31.0} & 5.9M (1.9$\times$) \\
            HIVE  & 53.8 & 17.5 & 45.5 & 49.1 & 20.5 & 20.0 & 34.4 & \best{4.8M (2.4$\times$)} \\
            \midrule

            \multicolumn{9}{c}{{\textit{Qwen3-4B-Base}}} \\
            \midrule
            DS    & 84.4 & 22.5 & 59.5 & 70.1 & 38.2 & 49.7 & {54.1} & 3.0M (1.0$\times$) \\
            Pre-filter & 80.3 & 19.2 & 51.8 & 67.4 & 38.7 & 46.8 & {50.7} & 1.6M (1.8$\times$) \\
            GRESO & 83.5 & 21.8 & 59.6 & 70.8 & 39.9 & 49.5 & {54.2} & 1.6M (1.8$\times$) \\
            PCL   & 80.5 & 18.3 & 52.1 & 67.5 & 39.8 & 46.5 & {50.7} & 1.6M (1.8$\times$) \\
            HIVE  & 84.7 & 22.5 & 59.9 & 71.0 & 40.1 & 50.4 & {54.7} & \best{1.5M (2.0$\times$)} \\
            \midrule

            \multicolumn{9}{c}{{\textit{Qwen2.5-Math-7B}}} \\
            \midrule
            DS    & 82.9 & 34.2 & 79.2 & 71.7 & 35.4 & 43.6 & 57.8 & 13.1M (1.0$\times$)\\
            Pre-filter & 76.8 & 34.2 & 61.4 & 62.4 & 33.6 & 41.3 & {51.6} & 6.3M (2.1$\times$) \\
            GRESO & 93.2 & 30.0 & 78.9 & 70.2 & 35.2 & 44.1 & 58.6 & 6.3M (2.1$\times$)\\
            PCL   & 76.4 & 32.5 & 62.7 & 65.1 & 34.4 & 40.8 & {52.0} & 6.3M (2.1$\times$) \\
            HIVE  & 93.0 & 36.7 & 79.2 & 70.8 & 34.8 & 43.8 & 59.7 & \best{3.9M (3.4$\times$)} \\
            \bottomrule
            \end{tabular}
        }
    \end{subtable}
    \hfill %
    \begin{subtable}[t]{0.324\textwidth}
        \centering
        \caption{Training time (hours).}
        \label{tab:efficiency_cost} %

        \resizebox{\linewidth}{!}{
            \begin{tabular}{c|ccc|c}
            \toprule
            Method & Train & Other & Rollout & Total \\ %
            \midrule

            \multicolumn{5}{c}{{\textit{Qwen2.5-Math-1.5B}}} \\
            \midrule
            DS    & 10.9 & 6.0 & 77.3 & 94.2 \\
            Pre-filter & 11.0 & 6.0 & 45.9 & 62.9 \\
            GRESO & 11.3 & 6.4 & 46.0 & 63.7 \\
            PCL   & 11.4 & 7.0 & 46.3 & 64.7 \\
            HIVE  & 11.2 & 6.2 & \best{42.1} & \best{59.5} \\
            \midrule

            \multicolumn{5}{c}{{\textit{DeepSeek-R1-Distill-Qwen-1.5B}}} \\
            \midrule
            DS    & 9.6  & 6.1 & 68.7 & 84.4 \\
            Pre-filter & 10.1 & 6.2 & 45.8 & 62.1 \\
            GRESO & 10.4 & 6.6 & 45.9 & 62.9 \\
            PCL   & 10.5 & 7.1 & 46.0 & 63.6 \\
            HIVE  & 10.1 & 6.7 & \best{42.7} & \best{59.5} \\
            \midrule

            \multicolumn{5}{c}{{\textit{Llama3.2-3b-Instruct}}} \\
            \midrule
            DS    & 20.6 & 10.2 & {145.0} & 175.8 \\
            Pre-filter & 21.4 & 10.0 & 57.3 & 88.7 \\
            GRESO & 21.8 & 10.5 & 57.5 & 89.8 \\
            PCL   & 22.0 & 11.2 & 57.6 & 90.8 \\
            HIVE  & 22.2 & 10.6 & \best{46.9} & \best{79.7} \\
            \midrule

            \multicolumn{5}{c}{{\textit{Qwen3-4B-Base}}} \\
            \midrule
            DS    & 9.7  & 6.2 & 86.0 & 101.9 \\
            Pre-filter & 10.2 & 6.3 & 45.9 & 62.4 \\
            GRESO & 10.5 & 6.7 & 46.1 & 63.3 \\
            PCL   & 10.6 & 7.3 & 46.1 & 64.0 \\
            HIVE  & 10.2 & 6.8 & \best{42.8} & \best{59.8} \\
            \midrule

            \multicolumn{5}{c}{{\textit{Qwen2.5-Math-7B}}} \\
            \midrule
            DS    & 32.1 & 12.6 & 153.7 & {198.4} \\
            Pre-filter & 32.4 & 12.9 & 66.1 & 111.4 \\
            GRESO & 32.8 & 13.2 & 66.3 & {112.3} \\
            PCL   & 33.0 & 14.0 & 66.5 & 113.5 \\
            HIVE  & 32.7 & 12.9 & \best{40.2} & \best{85.8} \\
            \bottomrule
            \end{tabular}
        }
    \end{subtable}
    \vskip -0.2in
\end{table*}

\section{Experiments}
\label{sec:experiments}
\setcounter{topnumber}{2}
We evaluate HIVE along the edge-tracking objectives introduced in Section~\ref{sec:method}. First, does HIVE preserve reasoning accuracy while using fewer rollouts? Second, is HIVE cheap enough to be practical? Third, does HIVE keep rejecting stale or low-utility prompts as the learning edge moves?

\subsection{Experimental Settings}
\best{Models and datasets.} We evaluate Qwen2.5-Math-1.5B/7B~\cite{yang2024qwen25mathtechnicalreportmathematical}, DeepSeek-R1-Distill-Qwen-1.5B~\cite{guo2025deepseek}, Qwen3-4B-Base~\cite{yang2025qwen3technicalreport}, Qwen2.5-14B/32B~\cite{qwen2.5-large}, and Llama-3.2-3B-Instruct~\cite{grattafiori2024llama3}. We use a context length of 4K for Qwen2.5-Math-1.5B/7B, 8K for DeepSeek-R1-Distill-Qwen-1.5B, Qwen3-4B-Base, and Llama-3.2-3B-Instruct, and 16K for Qwen2.5-14B/32B. We train on two math-reasoning corpora following~\cite{zheng2025act}: DAPO+MATH (DM), which combines DAPO~\cite{yu2025dapo} and MATH~\cite{hendrycks2021measuringmathematicalproblemsolving}, and a 30K subset of Open-R1~\citet{openr1}.

\best{Training and evaluation.} HIVE is implemented with verl~\cite{verl2025} and vLLM~\cite{vLLM2023SOSP}. Main experiments use 8$\times$A100-80G GPUs and train each method for {1000 steps}. We evaluate on six mathematical reasoning benchmarks: Math500~\cite{hendrycks2021measuringmathematicalproblemsolving}, AIME24~\cite{aops2024aime}, AMC~\cite{aops2024amc}, Minerva Math~\cite{NEURIPS2022_18abbeef}, Gaokao~\cite{zhang2024gaokao}, and Olympiad Bench~\cite{he-etal-2024-olympiadbench}. Following~\cite{zheng2025act}, we report {the checkpoint with the best average score on the six benchmarks}. Appendix~\ref{appendix-sec:detailed-experimental-setting} gives full training details.

\best{Baselines.} We compare against Dynamic Sampling (DS)~\cite{yu2025dapo}, a rollout-based online filter; Pre-filter~\cite{Gao2025PCL}, a fixed offline prompt filter applied before RL training; GRESO~\cite{zheng2025act}, a history-based selector; and PCL~\cite{Gao2025PCL}, an auxiliary prompt-curriculum selector. Since the original PCL evaluation uses a fixed wall-clock budget, whereas our main protocol fixes training at 1000 steps, we evaluate Pre-filter and PCL with a conservative two-run protocol. For the rollout budget in Table~\ref{tab:main_comparison}, we run each method once using the original PCL settings and once using our common settings, and report the higher average benchmark score. %

\subsection{Main Results: Accuracy Under Lower Rollout Cost}
\label{subsec:main-experiment}
\begin{table}[t]
\centering
\begin{minipage}{0.86\textwidth}
\centering
\small
\captionsetup{font=small}
\caption{Performance (\%) and rollout comparison of Qwen2.5-Math-1.5B and Qwen2.5-Math-7B trained on Open-R1.}
\label{tab:openr1-main-exp}

\renewcommand{\arraystretch}{0.90}
\resizebox{\linewidth}{!}{%
\begin{tabular}{c|cccccc|c|c}
\toprule
{Method} & {Math500} & {AIME24} & {AMC} & {Gaokao} & {Miner.} & {Olymp.} & {Avg.} & {\#Rollout} \\
\midrule

\multicolumn{9}{c}{{\textit{Qwen2.5-Math-1.5B}}} \\
\midrule
DS    & 77.1 & 16.7 & 50.3 & 65.5 & 30.9 & 39.7 & 46.7 & 3.8M (1.0$\times$) \\
GRESO & 76.1 & 20.0 & 50.6 & 65.1 & 30.0 & 39.2 & 46.8 & 1.6M (2.4$\times$)\\
HIVE  & 76.9 & 18.3 & 48.7 & 65.4 & 30.1 & 40.3 & 46.6 & \best{{1.5M} (2.5$\times$)} \\
\midrule

\multicolumn{9}{c}{{\textit{Qwen2.5-Math-7B}}} \\
\midrule
DS    & 82.8 & 34.2 & 63.5 & 67.3 & 35.7 & 46.3 & 55.0 & 11.4M (1.0$\times$) \\
GRESO & 82.3 & 35.0 & 64.4 & 66.8 & 36.5 & 45.7 & 55.1 & 3.4M (3.4$\times$) \\
HIVE  & 82.6 & 36.7 & 63.9 & 66.7 & 36.2 & 45.9 & 55.3 & \best{2.5M (4.6$\times$)} \\
\bottomrule
\end{tabular}
}
\end{minipage}
\vspace{-8pt}
\end{table}

\best{HIVE preserves accuracy while reducing rollouts.}
We compare HIVE against Dynamic Sampling (DS)~\cite{yu2025dapo}, a rollout-based online filter, and GRESO~\cite{zheng2025act}, a history-based selector. Table~\ref{tab:model_performance} reports results on DAPO+MATH across five model settings. HIVE consistently uses fewer rollouts than both baselines while maintaining comparable average accuracy. On Qwen2.5-Math-7B, HIVE reduces rollouts by 70\% compared with DS (13.1M to 3.9M) and achieves the highest average accuracy (59.7\% vs. 57.8\% for DS and 58.6\% for GRESO), reaching comparable accuracy with substantially fewer generated responses.

\newsavebox{\hiveScaleAnalysisBox}
\newsavebox{\hiveScaleCaptionBox}
\newlength{\hiveScaleHeight}
\newlength{\hiveScalePlotHeight}
\sbox{\hiveScaleAnalysisBox}{%
\begin{minipage}[b]{0.60\textwidth}
\textbf{The gains transfer across datasets and model scales.}
\mbox{Table~\ref{tab:openr1-main-exp}} evaluates Qwen2.5-Math-1.5B/7B on Open-R1. HIVE again keeps comparable or better average accuracy while reducing rollouts, reaching 4.6$\times$ fewer rollouts than DS on Qwen2.5-Math-7B. Figure~\ref{fig:scale_14b_32b} extends the comparison to Qwen2.5-14B and Qwen2.5-32B. HIVE follows a similar convergence trajectory to DS and GRESO while requiring fewer rollouts; for Qwen2.5-14B over 800 steps, HIVE uses 1.56M rollouts compared with 1.74M for GRESO and 3.45M for DS.
\par\vspace{0pt}
\end{minipage}%
}
\setlength{\hiveScaleHeight}{\dimexpr\ht\hiveScaleAnalysisBox+\dp\hiveScaleAnalysisBox\relax}
\noindent\begin{minipage}{\textwidth}
\usebox{\hiveScaleAnalysisBox}\hfill
\begin{minipage}[b][\hiveScaleHeight][s]{0.36\textwidth}
\centering
\sbox{\hiveScaleCaptionBox}{%
\begin{minipage}[b]{\linewidth}
\captionsetup{justification=raggedright,singlelinecheck=false,hypcap=false,aboveskip=6pt,belowskip=0pt}
\captionof{figure}{Scaling up to larger models, Qwen2.5-14B/32B.}
\label{fig:scale_14b_32b}
\par\vspace{-0.84\baselineskip}
\end{minipage}%
}
\setlength{\hiveScalePlotHeight}{\dimexpr\hiveScaleHeight-\ht\hiveScaleCaptionBox-\dp\hiveScaleCaptionBox\relax}
\includegraphics[height=\hiveScalePlotHeight,width=\linewidth,keepaspectratio]{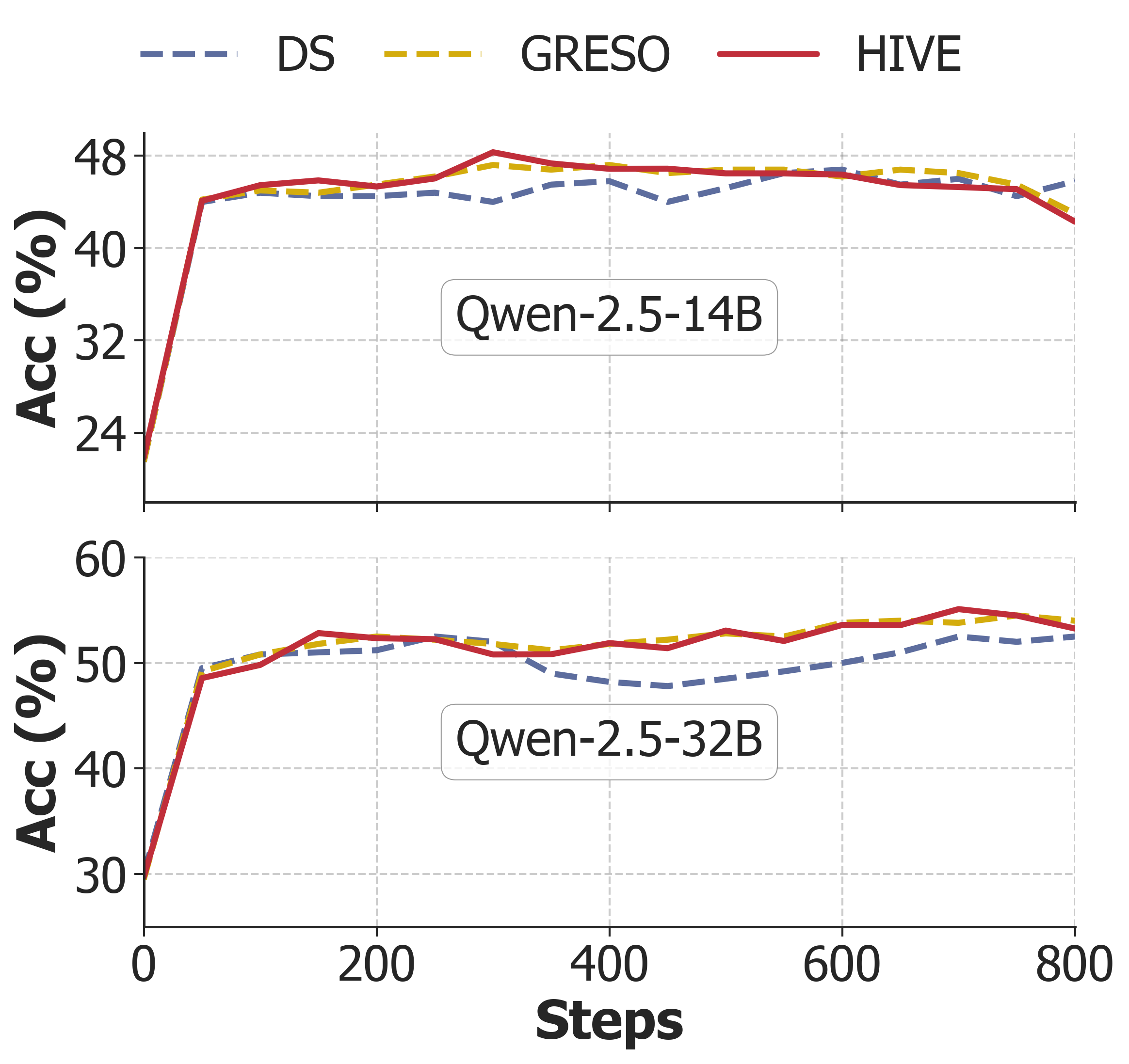}%
\par\nointerlineskip
\usebox{\hiveScaleCaptionBox}%
\par\vspace{0pt}

\end{minipage}
\end{minipage}
\par

\best{Wall-clock savings come from rollout reduction, not hidden overhead.}
Table~\ref{tab:efficiency_cost} decomposes end-to-end training time. HIVE reduces rollout time by up to 3.8$\times$ and total training time by up to 2.3$\times$ over DS. On Qwen2.5-Math-7B, total training time drops from 198.4 hours (DS) and 112.3 hours (GRESO) to 85.8 hours (HIVE), while rollout time drops from 153.7 and 66.3 to 40.2.

\noindent\begin{minipage}{\textwidth}
\begin{wraptable}[9]{r}{0.34\textwidth}
    \vspace{-1.2em}
    \centering
    \caption{Performance comparison (\%) on tool calling tasks.}
    \vspace{-0.6em}
    \label{tab:tool_calling}
    \small
    \setlength{\tabcolsep}{3.2pt}
    \renewcommand{\arraystretch}{0.92}
    \begin{tabular}{lcc}
        \toprule
        Method & BFCL-V2 & \#Rollout \\
        \midrule
        DS & 79.08 & 5.3M \\
        Pre-filter & 77.04 & 1.4M \\
        GRESO & 78.96 & 1.4M \\
        PCL & 76.84 & 1.4M \\
        HIVE & \best{79.26} & \best{1.1M} \\
        \bottomrule
    \end{tabular}
    \vspace{-0.6em}
\end{wraptable}
\noindent\best{Performance on tool calling tasks.} Besides, we also verify the effectiveness of HIVE on tool calling tasks. We train Qwen2.5-1.5B on 80\% prompts samples from BFCL-V2~\cite{gorilla-openfunctions-v2} and evaluate the performance on the remaining 20\% subset~\cite{guo2025sample}. As shown in Table~\ref{tab:tool_calling}, HIVE achieves the highest BFCL-V2 score while using only 1.1M rollouts instead of 5.3M. This corresponds to a 4.8$\times$ rollout reduction relative to DS. Under the same low-rollout regime, HIVE also outperforms the static Pre-filter baseline, the history-only GRESO selector, and the auxiliary PCL selector. These results suggest that moving-edge tracking is not only specific to mathematical reward signals, but also applicable on tool-calling tasks, where HIVE continues to concentrate rollout compute on informative prompts.
\par
\end{minipage}
\par

\noindent\begin{minipage}{\textwidth}
\begin{wraptable}[10]{r}{0.40\textwidth}
    \vspace{-0.8em}
    \centering
    \captionsetup{justification=raggedright,singlelinecheck=false,hypcap=false}
    \caption{Sudoku with non-binary rewards. Exact-solve rate (avg@4); one run per method.}
    \label{tab:sudoku}
    \small
    \setlength{\tabcolsep}{3.2pt}
    \renewcommand{\arraystretch}{0.92}
    \begin{tabular}{lcc}
        \toprule
        Method & Exact-solve (\%) & \#Rollout \\
        \midrule
        GRESO & 61.1 & 1.38M \\
        PCL & 61.5 & 1.68M \\
        HIVE & \best{62.7} & \best{1.15M} \\
        \bottomrule
    \end{tabular}
    \vspace{-0.8em}
\end{wraptable}
\noindent\textbf{Extension to non-binary verifiable rewards.}
We further train Qwen2.5-Math-1.5B with GRPO ($G=8$) on $4\times4$ Sudoku, a constraint-satisfaction task with a reward adapted from Reasoning Gym~\cite{stojanovski2025reasoninggym}: an exact solution receives $1.0$, and other answers receive $0.9m/16$, where $m$ is the number of matching cells. Table~\ref{tab:sudoku} shows that HIVE reaches a 62.7\% exact-solve rate with 1.15M rollouts, using 16.7\% fewer rollouts than GRESO and 31.5\% fewer than PCL, while improving the observed solve rate by 1.6 and 1.2 percentage points, respectively. These results provide additional evidence that HIVE's accuracy--cost tradeoff extends to a verifiable reward with partial credit. The comparison uses one run per method; DS is omitted due to the compute budget.
\par
\end{minipage}
\par

\subsection{Mechanism Analysis: Efficiency and Moving-Edge Adaptation}
\label{subsec:further-experiment}
\begin{figure*}[t]
    \centering
    \begin{subfigure}[b]{0.24\textwidth}
        \centering
        \includegraphics[width=\textwidth]{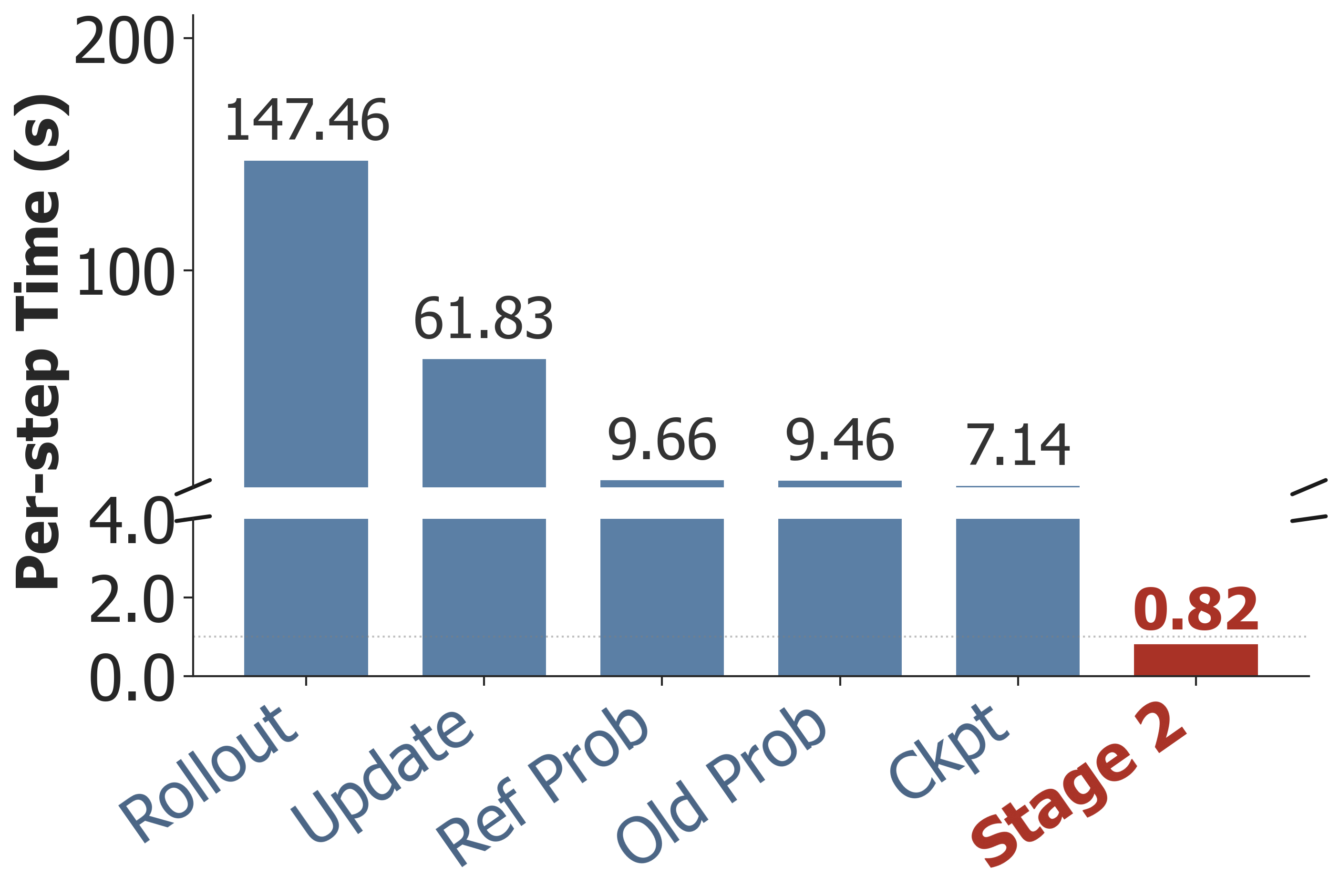}
        \caption{{Train-Time Breakdown}}
        \label{fig:per_step_time_analysis}
    \end{subfigure}
    \begin{subfigure}[b]{0.24\textwidth}
        \centering
        \includegraphics[width=\textwidth]{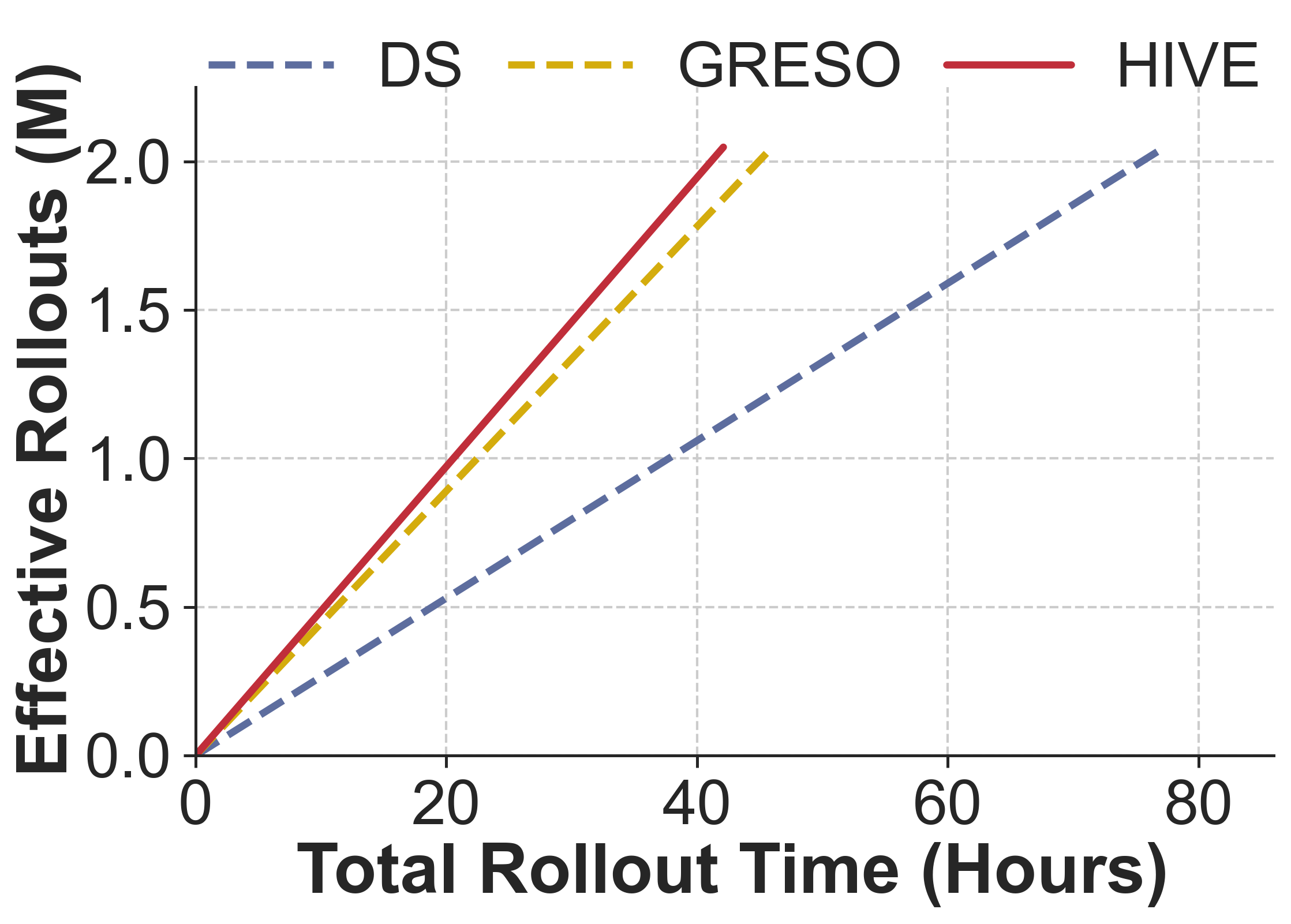}
        \caption{Efficiency Comparison}
        \label{fig:total_rollout}
    \end{subfigure}
    \hfill
    \begin{subfigure}[b]{0.24\textwidth}
        \centering
        \includegraphics[width=\textwidth]{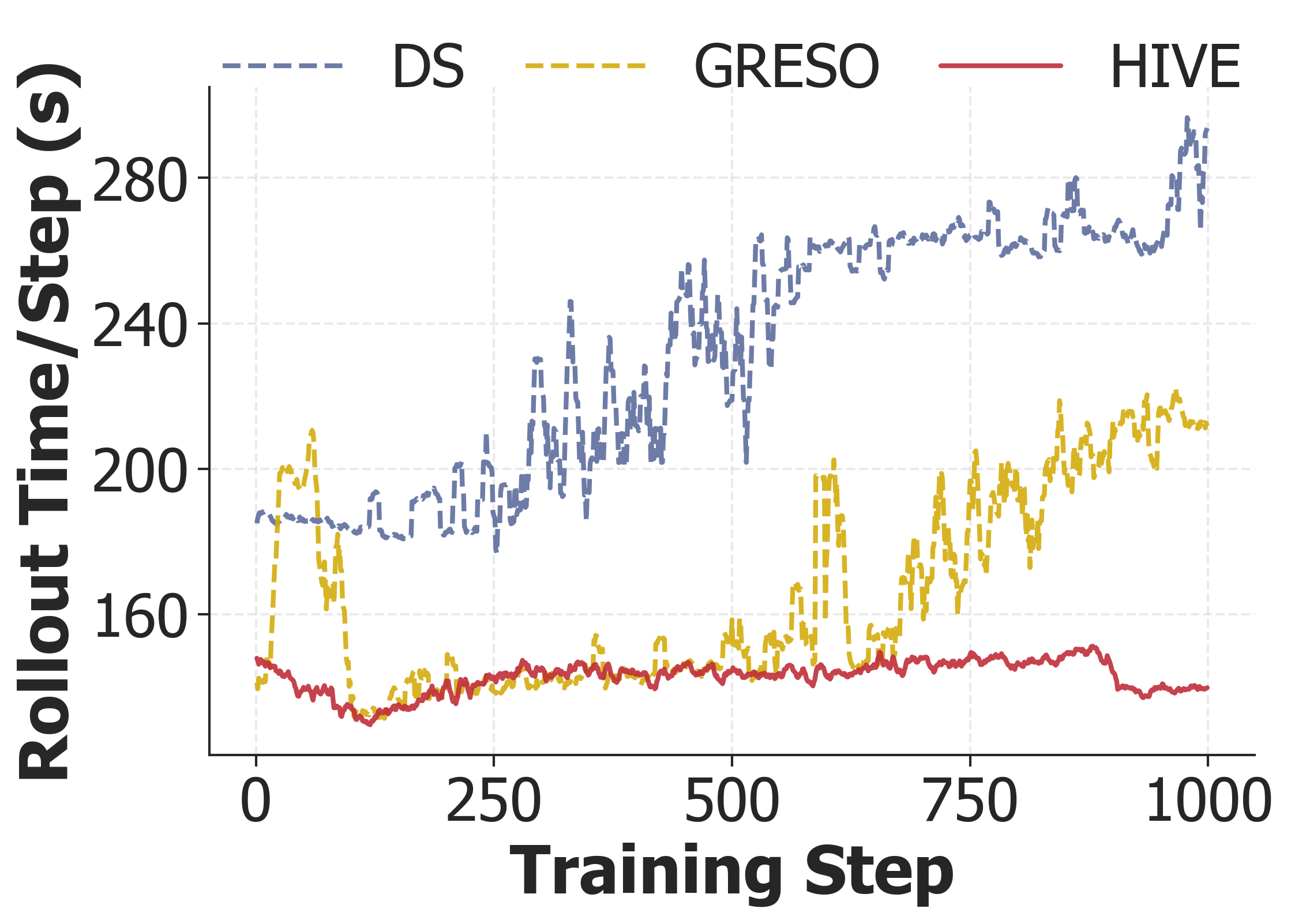}
        \caption{Generation Time}
        \label{fig:gen_time}
    \end{subfigure}
    \hfill
    \begin{subfigure}[b]{0.24\textwidth}
        \centering
        \includegraphics[width=\textwidth]{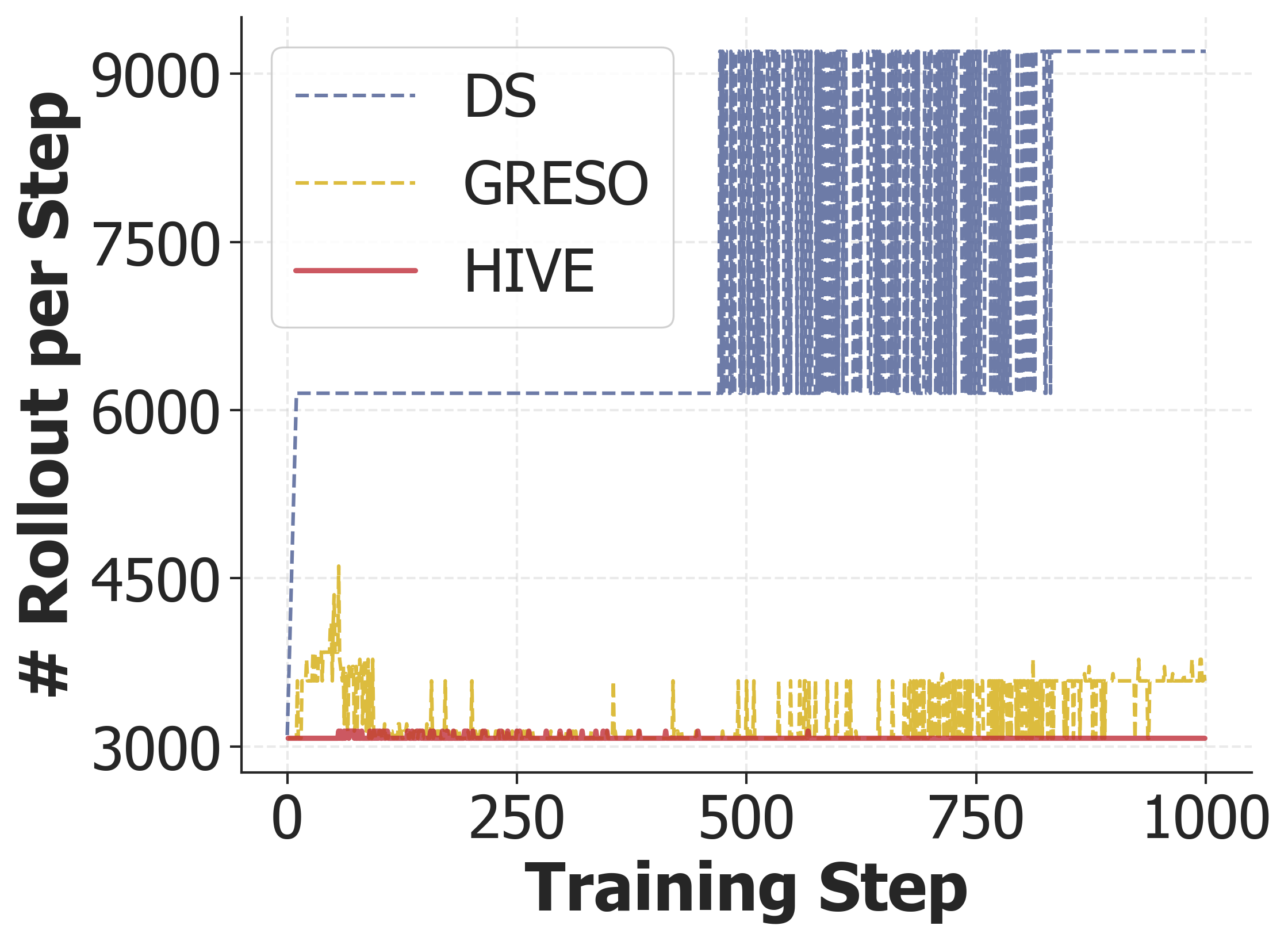}
        \caption{Rollouts per Step}
        \label{fig:rollout_step}
    \end{subfigure}
    \hfill

    \caption{Efficiency analysis of Qwen2.5-Math-1.5B trained on DAPO+MATH. (a) Stage~2 prompt-side verification adds little time relative to rollout generation. (b) Lines show average effective-rollout throughput using the rollout times in Table~\ref{tab:efficiency_cost}. (c) HIVE maintains lower generation latency throughout training. (d) HIVE rolls out fewer prompt groups per step, concentrating compute on verified candidates.}
    \label{fig:further-analysis}
\end{figure*}

\best{Online verification is cheap enough to run before rollouts.}
Figure~\ref{fig:per_step_time_analysis} isolates the per-step overhead of Stage~2. Computing prompt-side entropy adds only 0.82 seconds, less than 0.4\% of the iteration time, while the rollout phase alone takes 147.46 seconds. This supports the computational efficiency objective: HIVE gains online information without paying for response generation during selection.

\best{HIVE improves effective rollout throughput.}
Figure~\ref{fig:total_rollout} shows that HIVE accumulates effective rollouts per {rollout-generation hour} fastest. Figure~\ref{fig:gen_time} further shows that HIVE maintains low rollout-generation latency throughout training, while GRESO becomes slower later in training. This is consistent with metadata staleness: a history-only selector increasingly admits prompts that look useful in the log but no longer produce efficient current updates. Figure~\ref{fig:rollout_step} shows that HIVE also generates fewer rollout groups per step, concentrating compute on a smaller set of verified prompts near the current edge.

\begin{figure*}[t]
    \centering
    \begin{subfigure}[b]{0.24\textwidth}
        \centering
        \includegraphics[width=\textwidth,trim=0 0 440.4bp 0,clip]{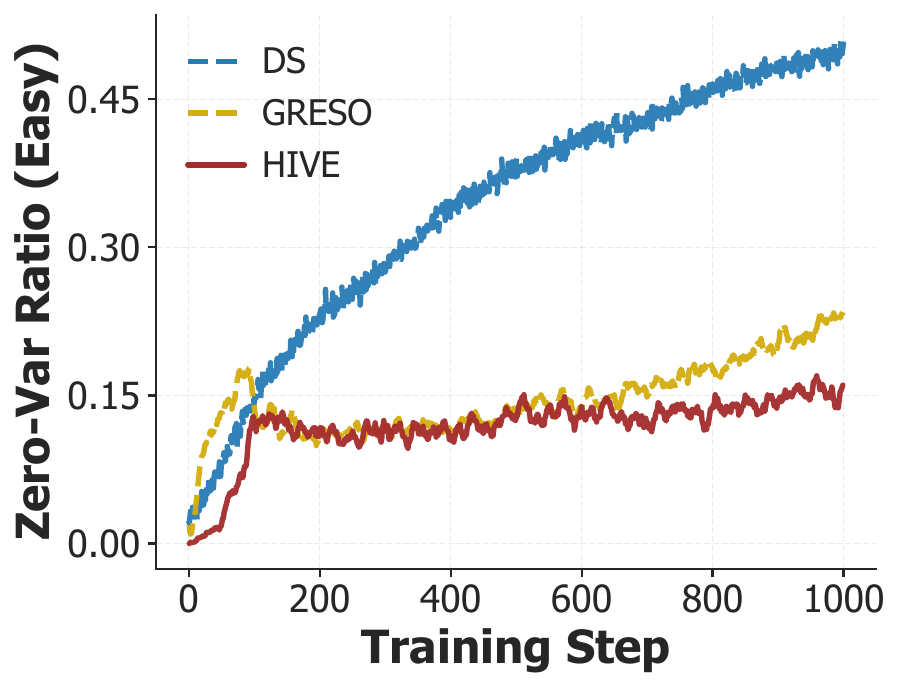}
        \caption{Easy zero-variance}
        \label{fig:zerovar_easy}
    \end{subfigure}
    \hfill
    \begin{subfigure}[b]{0.24\textwidth}
        \centering
        \includegraphics[width=\textwidth,trim=440.4bp 0 0 0,clip]{figures/experiment/fig6-zerovar-easy-hard.pdf}
        \caption{Hard zero-variance}
        \label{fig:zerovar_hard}
    \end{subfigure}
    \hfill
    \begin{subfigure}[b]{0.24\textwidth}
        \centering
        \includegraphics[width=\textwidth]{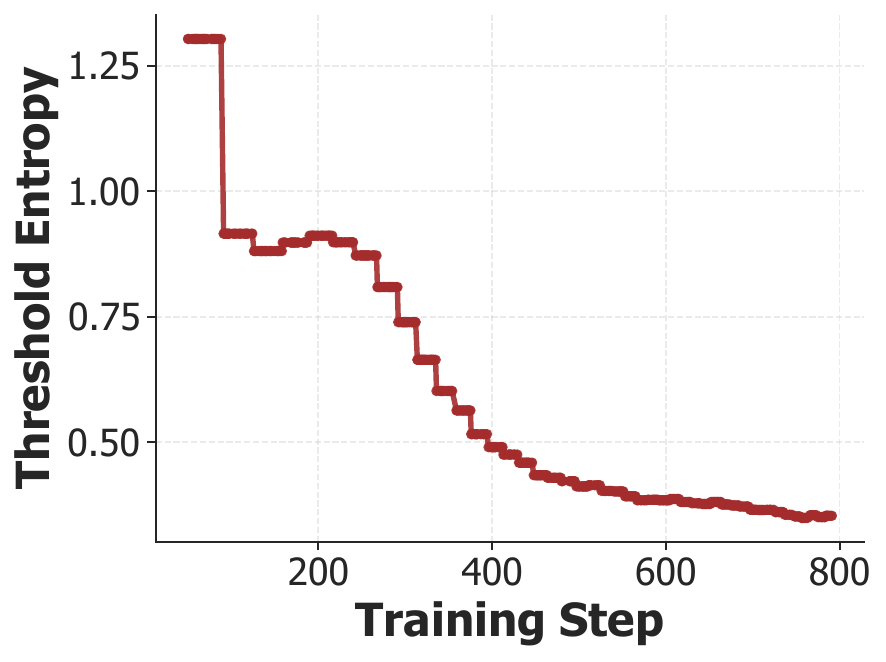}
        \caption{Entropy threshold}
        \label{fig:gamma_entropy}
    \end{subfigure}
    \hfill
    \begin{subfigure}[b]{0.24\textwidth}
        \centering
        \includegraphics[width=\textwidth]{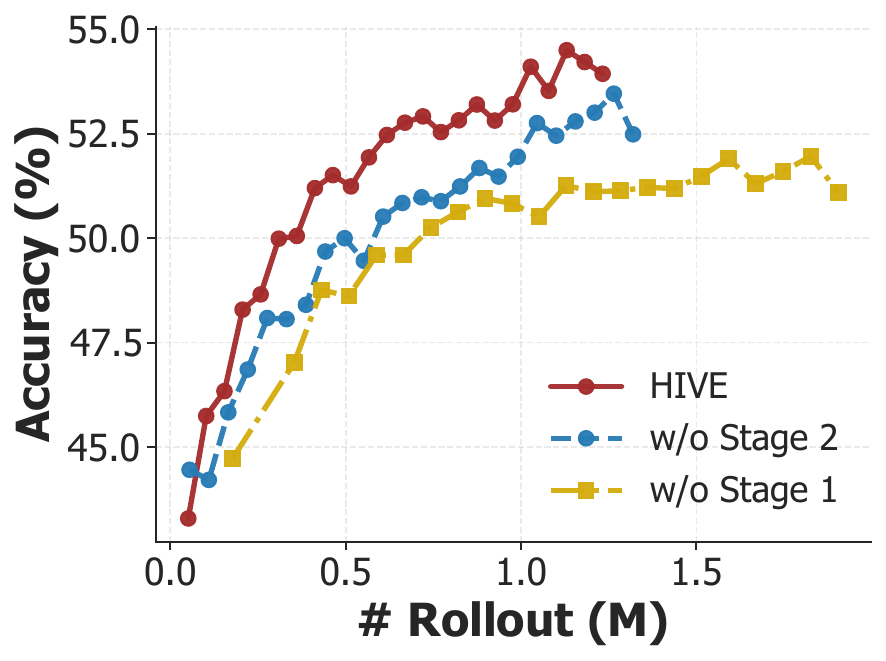}
        \caption{Stage ablation}
        \label{fig:ablation}
    \end{subfigure}

    \caption{\textbf{Zero-variance ratios and component analysis.}
    (a)~Easy zero-variance ratio: HIVE keeps the lowest fraction of all-correct prompts throughout training.
    (b)~Hard zero-variance ratio: HIVE rapidly drops below DS and tracks GRESO.
    (c)~The Stage~2 median prompt-entropy threshold $\gamma$ adapts as the policy becomes more confident.
    (d)~Removing either Stage~1 or Stage~2 worsens the accuracy--rollout tradeoff.}%
    \label{fig:non-zero-diveregence-comparison}
    \vskip -0.08in
\end{figure*}

\best{HIVE lowers the observed zero-variance ratio.}
The zero-variance ratio measures how often selected prompts produce no informative gradient. Figures~\ref{fig:zerovar_easy} and~\ref{fig:zerovar_hard} show that DS wastes a large fraction of rollouts on easy zero-variance prompts as the model improves, while GRESO remains vulnerable to stale historical records. HIVE maintains a lower zero-variance ratio for both easy and hard prompts, indicating that online verification helps track the moving edge.%

\subsection{Component Study: Why Both Stages Matter}
\label{subsec:exp-component-study}

\best{The entropy gate adapts with the current policy.}
Figure~\ref{fig:gamma_entropy} tracks the Stage~2 median threshold $\gamma$, which decreases as the model becomes more confident. This behavior is expected: as the policy
improves, the prompt-side uncertainty distribution shifts, and a fixed threshold would either overprune or admit stale prompts. The median gate keeps the verification rule relative to the current candidate pool. \best{Both stages contribute in the reported ablation.} Figure~\ref{fig:ablation} compares the full HIVE pipeline with variants that remove either stage. Without Stage 2, HIVE loses online precision and becomes more vulnerable to stale metadata. Without Stage 1, the verifier receives too many low-value candidates, reducing efficiency. The full method gives the best accuracy–rollout tradeoff, supporting the decomposition into efficient historical narrowing followed by online verification.

\begin{wraptable}[10]{r}{0.48\textwidth}
    \vspace{-1.0em}
    \centering
    \captionsetup{justification=raggedright,singlelinecheck=false}
    \caption{Stage~2 score comparison on Qwen2.5-Math-1.5B / DAPO+MATH.}
    \label{tab:surface-feature-controls}
    \small
    \setlength{\tabcolsep}{3.2pt}
    \renewcommand{\arraystretch}{0.92}
    \begin{tabular}{lcc}
        \toprule
        Stage~2 signal & Avg. (\%) & \#Rollout \\
        \midrule
        Prompt length & 45.05 & 2.9M \\
        Rare-token ratio & 44.52 & 3.0M \\
        Math-symbol density & 43.57 & 2.8M \\
        Format/tokenization & 44.66 & 2.9M \\
        Prompt entropy & \best{48.80} & 3.1M \\
        \bottomrule
    \end{tabular}
    \vspace{-1.0em}
\end{wraptable}
\noindent\textbf{Ablation on Stage~2 score.}
We hold Stage~1 and the GRPO settings fixed and replace only the Stage~2 score with prompt length, rare-token ratio, math-symbol density, or a formatting/tokenization score. Appendix~\ref{appendix-subsec:surface-controls} defines these controls. As shown in Table~\ref{tab:surface-feature-controls}, the alternatives reach 43.57--45.05\% average accuracy, compared with 48.80\% for prompt entropy. Their rollout counts are close but not identical (2.8--3.1M), with HIVE using the most rollouts in this comparison. The observed 3.75--5.23-point gaps suggest that these surface-based rankings do not recover the entropy gate's benefit under similar rollout costs. This supports the choice of prompt entropy over the tested substitutes, without ruling out correlations between entropy and surface properties.

\section{Related Work}
\label{sec:related-work}

\textbf{RLVR for LLM reasoning.}
Reinforcement learning is a standard post-training tool for aligning and improving LLMs~\cite{christiano2017deep,bai2022training,ouyang2022training,rafailov2023direct,yang2026seeing}. Recent reasoning systems increasingly rely on reinforcement learning with verifiable rewards (RLVR), where correctness can be checked directly for tasks such as mathematics and code~\cite{guo2025deepseek,hu2025reinforce++,liu2503understanding,jaech2024openai,Zhang2025SRPOAC}. Algorithmic work improves this training loop through value-enhanced PPO variants such as VC-PPO~\cite{yuan2025s} and VAPO~\cite{yue2025vapo}, group-based methods such as RLOO~\cite{ahmadian2024back} and GRPO~\cite{shao2024deepseekmath}, replay-based updates~\cite{liang2025squeeze,li2025repo,zhang2025rlep}, and objective modifications such as Dr. GRPO~\cite{liu2503understanding} and GSPO~\cite{zheng2025group}. These methods improve optimization stability or sample reuse, but they do not directly solve the pre-rollout allocation problem studied here: which prompts deserve expensive response generation at the current training step.

\textbf{Prompt selection and data-efficient RL training.}
Data curation is essential for efficient LLM training. Offline filtering methods rank prompts before training using various heuristic metrics~\cite{Wang2025ReinforcementLF,Fatemi2025ConciseRV,Li2025LIMRLI,shi2025efficient,tang2025towards,zhou2023lima,ye2025limo}. They are efficient once training begins, but they are policy-agnostic and cannot adapt as the model's competence changes. Rollout-based online methods, including Dynamic Sampling in DAPO~\cite{yu2025dapo} and related online down-sampling or difficulty-filtering approaches~\cite{meng2025mm,foster2025learning,xu2025not,lin2025cppo,sun2025efficient,bae2025online}, are more precise yet more costly, requiring more rollouts. History-based and estimator-based approaches such as GRESO~\cite{zheng2025act}, Bayesian task selection~\cite{chen2025self,zeng2025cures,shen2025bots,mao2026dps}, and curriculum estimators use past signals to reduce selection cost, but their signals become stale as training continues. Auxiliary-model methods such as DOTS~\cite{sun2025dots} and PCL~\cite{Gao2025PCL} add external or on-policy predictors for difficulty, improving adaptivity at the cost of additional compute or memory.

HIVE combines history-based filtering with online prompt verification to preserve efficiency while correcting stale selections before rollout. See Appendix~\ref{appendix-sec:detailed_related_work} for detailed related work.

\section{Conclusion}
RLVR prompt selection is best understood as a moving-edge tracking problem. Useful prompts are not merely those that look difficult or uncertain in isolation; they are prompts that remain informative for the current policy. HIVE efficiently identifies those prompts with a dual-stage mechanism: history-informed narrowing reduces the candidate pool, and prompt-side online verification filters stale candidates under the current model. Across math reasoning benchmarks and model scales, HIVE reduces rollout and training cost while preserving or even improving final performance.

\section{Limitations and Broader Impacts}
\label{sec:limitations}
\paragraph{Limitations.}
Our rank-consistency analysis is intended as a conservative sufficient condition for using prompt-side entropy as an online verifier, rather than a tight characterization of every empirical prompt pair. The proof uses worst-case concentration and bounded-approximation arguments, which may yield loose margins for modern LRMs. This does not affect HIVE's algorithmic design or empirical conclusions, as the method does not rely on exact prompt-level ranking; it only requires prompt-side entropy to provide a useful low-cost signal for filtering stale candidates. This role is validated by our prompt--response entropy analysis, component ablations, and end-to-end efficiency gains. Tighter variance-dependent guarantees remain an interesting direction for future work.
\paragraph{Broader impacts.}
HIVE improves the efficiency of RL post-training by reducing unnecessary rollouts, thereby lowering training cost, energy consumption, and carbon footprint. This may make RLVR-style post-training more accessible to resource-constrained research groups. We do not foresee specific negative societal impacts beyond the general risks associated with stronger LLM capabilities.

\bibliographystyle{plainnat}
\bibliography{references}

\clearpage
\appendix
\begingroup
\raggedbottom
\setlength{\textfloatsep}{10pt plus 2pt minus 2pt}
\setlength{\intextsep}{10pt plus 2pt minus 2pt}
\setlength{\floatsep}{8pt plus 2pt minus 2pt}
\section*{Appendix}
In this appendix, we provide the rank-consistency analysis for Theorem~\ref{thm:ranking-informal} (Appendix~\ref{appendix-sec:proof-thm}), experimental settings (Appendix~\ref{appendix-sec:detailed-experimental-setting}), preliminary-study details (Appendix~\ref{sec:preliminary-setup}), detailed related work (Appendix~\ref{appendix-sec:detailed_related_work}), additional experiments (Appendix~\ref{appendix-sec:additional-exp}), the HIVE algorithm (Appendix~\ref{appendix-sec:alg}), and asset licenses (Appendix~\ref{sec:asset-licenses}).

\section{Conditional Rank-Consistency Analysis}
\label{appendix-sec:proof-thm}
This section proves the conditional statement in Theorem~\ref{thm:ranking-informal}. The result is meant to formalize the intuition behind Stage 2 under simplified assumptions. The main empirical support for the online verification part of HIVE comes from the correlation and ablation results reported in the paper.

\subsection{Observable Quantities and Sampling}
Fix a policy checkpoint $\theta$ and sampling policy $q(\cdot\mid x)$. Let $\mathcal V$ be the finite vocabulary and use natural logarithms. For a prompt with $L_x\geq2$ scored positions, define
\begin{equation}
V(x)=\frac{1}{L_x-1}\sum_{\ell=2}^{L_x}
\mathcal H\!\left(p_\theta(\cdot\mid x_{<\ell})\right).
\end{equation}
For $n$ independent rollouts $y^{(r)}\sim q(\cdot\mid x)$ with nonempty responses of lengths $L_r$, define
\begin{equation}
\begin{aligned}
U^{(r)}(x)&=\frac{1}{L_r}\sum_{\ell=1}^{L_r}
\mathcal H\!\left(p_\theta(\cdot\mid x,y^{(r)}_{<\ell})\right),\\
U^*(x)&=\mathbb E_{y\sim q(\cdot\mid x)}[U^{(r)}(x)],
\qquad \hat U_n(x)=\frac{1}{n}\sum_{r=1}^n U^{(r)}(x).
\end{aligned}
\end{equation}
The scoring distribution $p_\theta$ uses a fixed temperature convention, held constant within each comparison. %
\par\addvspace{0.5\baselineskip}
\noindent
\begin{minipage}[t]{0.60\linewidth}
\vspace{0pt}
\setlength{\parskip}{4pt}
\subsection{Explicit Calibration Assumption}
\label{subsec:thm-theoretical-proof}
To relate prompt-side predictive entropy $V(x)$ to expected response-side predictive entropy $U^*(x)$, we impose the following bounded affine calibration assumption for a fixed policy checkpoint and a specified prompt set. This assumption is meant to capture a linear relationship between the two quantities up to a bounded error, in part verified in the binned linearity shown in Figure~\ref{fig:entropy_correlation}, empirically.
\begin{assumption}[Bounded affine calibration]
\label{ass:calibration}
For the fixed checkpoint and a specified prompt set $\mathcal X$, there exist constants $a>0$, $b\in\mathbb R$, and $\varepsilon_{\mathrm{cal}}\geq0$ such that
\begin{equation}
\begin{aligned}
U^*(x)&=aV(x)+b+e(x),\\
|e(x)|&\leq\varepsilon_{\mathrm{cal}}
\quad\text{for every }x\in\mathcal X.
\end{aligned}
\end{equation}
\end{assumption}
This is an assumption about prompt and response entropies. %
\end{minipage}\hfill
\begin{minipage}[t]{0.375\linewidth}
\vspace{0pt}
\captionsetup{font=small,skip=4pt}
\centering
\includegraphics[width=\linewidth]{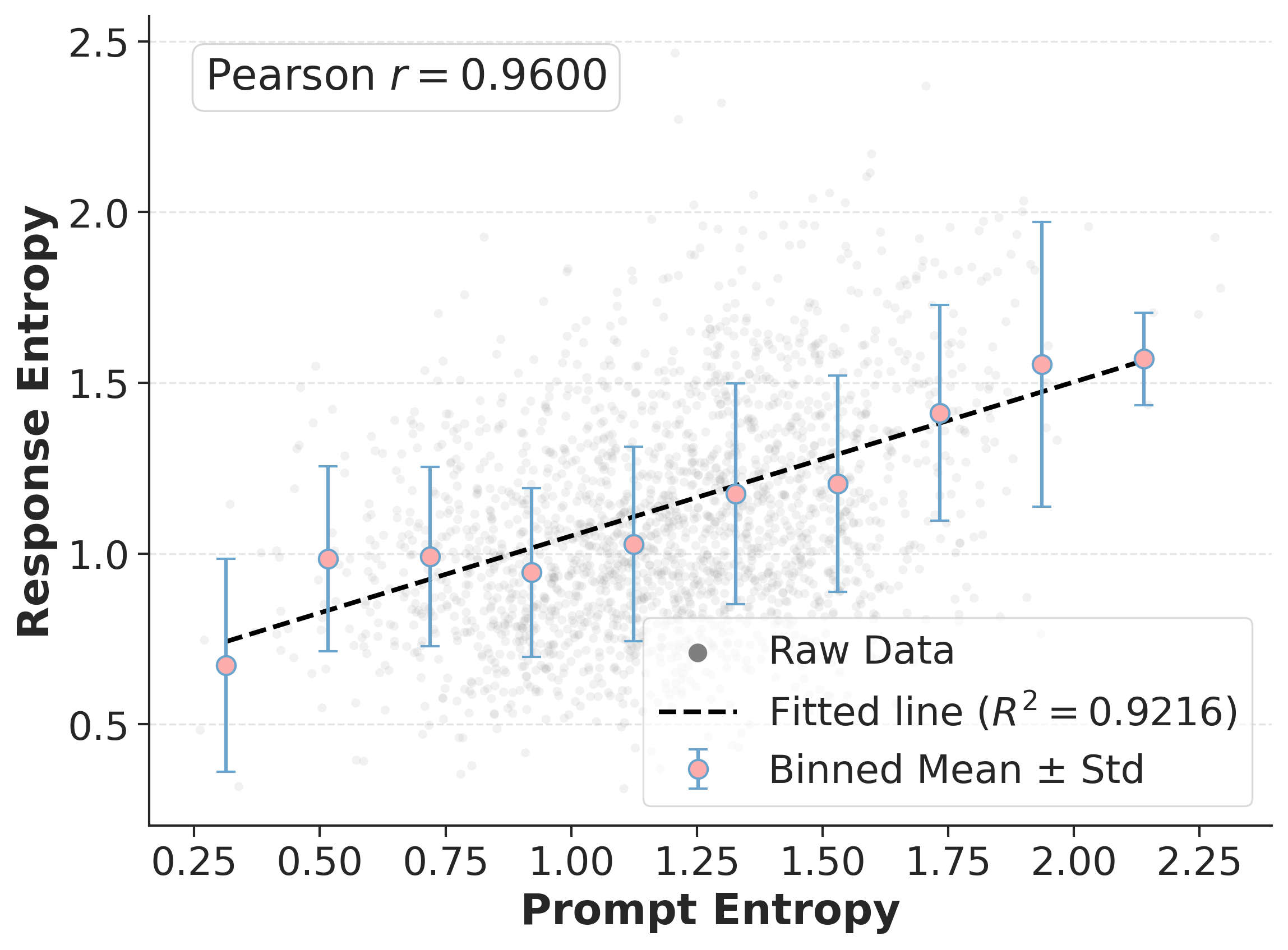}
\captionof{figure}{{\textbf{Binned prompt--response entropy association.} The plot summarizes 2,048 prompts; grey points are individual observations and blue points are bin means with standard deviations. The reported $r=0.9600$ and $R^2=0.9216$ describe the fitted bin means.}}
\label{fig:entropy_correlation}

\end{minipage}\par

\subsection{Finite-Rollout Error}
\begin{lemma}[Concentration across independent rollouts]
\label{lemma:concentration}
For a fixed prompt $x$ and $\eta>0$,
\begin{equation}
\Pr\!\left(|\hat U_n(x)-U^*(x)|>\eta\right)
\leq2\exp\!\left(-\frac{2n\eta^2}{(\log|\mathcal V|)^2}\right).
\label{eq:concentration_prob}
\end{equation}
\end{lemma}
\begin{proof}
Each token-distribution entropy lies in $[0,\log|\mathcal V|]$, so each response average lies in the same interval. Hoeffding's inequality~\cite{hoeffding1963probability} applies to the $n$ independent response averages.
\end{proof}

\subsection{Pairwise Rank Consistency}
\begin{theorem}[Conditional pairwise ordering]
\label{thm:ranking-appendix}
Fix two prompts $x,x'\in\mathcal X$ before drawing their diagnostic rollouts. Under Assumption~\ref{ass:calibration}, for any $\delta\in(0,1)$, define
\begin{equation}
\eta_n(\delta)=\log|\mathcal V|\sqrt{\frac{\log(4/\delta)}{2n}}.
\label{eq:concentration_bound}
\end{equation}
With probability at least $1-\delta$,
\begin{equation}
\left|[\hat U_n(x)-\hat U_n(x')]-a[V(x)-V(x')]\right|
\leq2\varepsilon_{\mathrm{cal}}+2\eta_n(\delta).
\label{eq:theorem_bound-appendix}
\end{equation}
Consequently, if $a|V(x)-V(x')|>2\varepsilon_{\mathrm{cal}}+2\eta_n(\delta)$, then on the same event
\begin{equation}
\operatorname{sign}(\hat U_n(x)-\hat U_n(x'))
=\operatorname{sign}(V(x)-V(x')).
\label{eq:theorem_order-appendix}
\end{equation}
For the expected entropies $U^*(x),U^*(x')$, the sufficient condition is $a|V(x)-V(x')|>2\varepsilon_{\mathrm{cal}}$.
\end{theorem}
\begin{proof}
Applying Lemma~\ref{lemma:concentration} at tolerance $\eta_n(\delta)$ gives failure probability at most $\delta/2$ for each prompt. A union bound therefore yields, with probability at least $1-\delta$,
\[
|[\hat U_n(x)-\hat U_n(x')]-[U^*(x)-U^*(x')]| \leq2\eta_n(\delta).
\]
The calibration assumption gives
\[
|[U^*(x)-U^*(x')]-a[V(x)-V(x')]|=|e(x)-e(x')|
\leq2\varepsilon_{\mathrm{cal}}.
\]
The triangle inequality proves Eq.~\eqref{eq:theorem_bound-appendix}. A margin exceeding this error preserves the sign because $a>0$. Omitting the sampling error gives the assertion for expected entropies.
\end{proof}

\section{Detailed Experimental Setting for Main Study}
\label{appendix-sec:detailed-experimental-setting}

This section gives the full experimental protocol for the main study. We follow the general setup of GRESO~\cite{zheng2025act} and make three aspects explicit: the shared training and evaluation protocol, the budget-matched baseline comparison, and the HIVE-specific selector configuration.

\subsection{Shared Training and Evaluation Protocol}
\label{appendix-subsec:shared-exp-protocol}

\paragraph{Models and datasets.}
We evaluate Qwen2.5-Math-1.5B~\cite{yang2024qwen25mathtechnicalreportmathematical}, DeepSeek-R1-Distill-Qwen-1.5B~\cite{guo2025deepseek}, Qwen3-4B-Base~\cite{yang2025qwen3technicalreport}, Qwen2.5-Math-7B~\cite{yang2024qwen25mathtechnicalreportmathematical}, and Llama-3.2-3B-Instruct~\cite{grattafiori2024llama3}. We additionally evaluate Qwen2.5-14B/32B~\cite{qwen2.5-large} in the scaling study. The context length is 4096 for Qwen2.5-Math-1.5B/7B, 8192 for DeepSeek-R1-Distill-Qwen-1.5B, Qwen3-4B-Base, and Llama-3.2-3B-Instruct, and 16384 for Qwen2.5-14B/32B. Training uses two math-reasoning corpora: DAPO+MATH, which combines DAPO~\cite{yu2025dapo} and MATH~\cite{hendrycks2021measuringmathematicalproblemsolving}, and a 30,000-example Open-R1 math subset~\cite{openr1}.

\paragraph{Optimization and rollout.}
Our implementation is based on verl~\cite{verl2025} and vLLM~\cite{vLLM2023SOSP}. We train the main models on 8$\times$A100 GPUs, use rollout sampling temperature 1.0, and sample $G=8$ responses per prompt. All models are trained for 1000 steps with AdamW~\cite{loshchilov2018decoupled}, learning rate $10^{-6}$, $\beta_1=0.9$, $\beta_2=0.999$, and weight decay 0.01. Following~\citet{yu2025dapo}, we remove the KL-divergence term. The actor is trained with Fully Sharded Data Parallel (FSDP)~\citep{zhao2023pytorch}.

For Qwen2.5-Math-1.5B/7B, Qwen3-4B-Base, and Llama-3.2-3B-Instruct, the final training batch contains {$B_{\mathrm{t}}=256$ prompts}, the {post-gate rollout target is $B_{\mathrm{cand}}=384$}, and the mini-batch size is 512. For DeepSeek-R1-Distill-Qwen-1.5B, we use $B_{\mathrm{t}}=128$, $B_{\mathrm{cand}}=192$, and mini-batch size 512. The larger Qwen2.5-14B/32B scaling runs use the same optimizer, sampling temperature, and selector hyperparameters, with the longer 16384-token context length noted above.

\paragraph{Rewards and evaluation.}
For reward assignment, we give a score of 0.1 when an answer can be successfully extracted and 1.0 when the extracted answer is correct. We evaluate every 50 steps on Math500~\cite{hendrycks2021measuringmathematicalproblemsolving}, AIME24~\cite{aops2024aime}, AMC~\cite{aops2024amc}, Minerva Math~\cite{NEURIPS2022_18abbeef}, Gaokao~\cite{zhang2024gaokao}, and Olympiad Bench~\cite{he-etal-2024-olympiadbench}. Following~\citet{zheng2025act}, we report the checkpoint with the best average performance across the six benchmarks.

\paragraph{Prompt template.}
All math-reasoning rollouts use the following template.

\begin{tcolorbox}[myexample={Question Template}]
Please solve the following math problem: \{\{Question Description\}\}. The assistant first thinks about the reasoning process step by step and then provides the user with the answer. Return the final answer in \textbackslash boxed\{\} tags, for example \textbackslash boxed\{1\}. Let's solve this step by step.
\end{tcolorbox}

\subsection{Baselines and Budget Matching}
\label{appendix-subsec:baselines-budget-matching}

We compare HIVE with four prompt-selection baselines. Dynamic Sampling (DS)~\cite{yu2025dapo} performs rollout-based online filtering and serves as the full-cost online baseline. Pre-filter is a fixed offline selector: prompts are filtered before RL training, and the selected subset remains unchanged during policy updates. GRESO~\cite{zheng2025act} is a history-based selector using past reward trajectories. PCL~\cite{Gao2025PCL} uses an auxiliary prompt-curriculum selector.

For each model family, GRESO, Pre-filter, and PCL are run under the rollout budgets shown in Table~\ref{tab:main_comparison}. Because the original PCL paper evaluates Pre-filter and PCL under a fixed wall-clock budget, while our main comparison fixes training at 1000 steps, we evaluate these two baselines with a matched-rollout two-run protocol. One run follows the original PCL-paper settings, and the other follows our common training settings; we report the higher average benchmark score across the two runs. Selector-specific runtime overhead is included in the ``Other'' column of Table~\ref{tab:efficiency_cost}.

\subsection{HIVE-Specific Hyperparameters and Initialization}
\label{appendix-subsec:hive-hyperparams}

Table~\ref{tab:hive-hyperparams} lists the selector hyperparameters used by HIVE, their notation in Section~\ref{sec:method} and Algorithm~\ref{alg:hive-training}, their deployed values, and their roles. Unless a row explicitly depends on $B_{\mathrm{t}}$, the same value is used across the main-result models and the Qwen2.5-14B/32B scaling study.

\begin{table}[t]
\centering
\footnotesize
\setlength{\tabcolsep}{4pt}
\caption{Complete list of HIVE-specific hyperparameters used in the main experiments. Defaults that we never override are marked with $^{\dagger}$.}
\label{tab:hive-hyperparams}
\begin{tabular}{>{\raggedright\arraybackslash}p{3.9cm} c c >{\raggedright\arraybackslash}p{6.1cm}}
\toprule
\textbf{Component} & \textbf{Symbol} & \textbf{Value} & \textbf{Description} \\
\midrule
\multicolumn{4}{l}{\textit{Stage~1: history-informed selection}} \\
\midrule
Reward / entropy mixing weight        & $\lambda$              & {$0.9$}      & Weight on the reward-history score in Eq.~\eqref{eq:stage1_select}; we set $\lambda{=}0.9$ in deployed runs and rely on Stage~2 to enforce the entropy criterion.\\
Initial easy exploration prob.        & $p_{e,\mathrm{easy}}^{(0)}$ & $0.5$ & Base retention probability for prompts whose recent visits were all-correct zero-variance groups. \\
Initial hard exploration prob.        & $p_{e,\mathrm{hard}}^{(0)}$ & $0.5$ & Base retention probability for prompts whose recent visits were all-incorrect zero-variance groups. \\
Probability step size                 & $\Delta p$              & $0.01$     & Per-step additive update applied to $p_{e,\tau}$ in Eq.~\eqref{eq:adaptive_pe}. \\
Lower clip on $p_{e,\tau}$            & $p_{\min}$              & $0.05$$^{\dagger}$  & Floor of the projection range $[p_{\min},p_{\max}]$ in Eq.~\eqref{eq:adaptive_pe}. \\
Upper clip on $p_{e,\tau}$            & $p_{\max}$              & $0.95$$^{\dagger}$  & Ceiling of the projection range $[p_{\min},p_{\max}]$ in Eq.~\eqref{eq:adaptive_pe}. \\
Reward-history score floor            & $\epsilon_{p}$          & $0.01$     & Per-prompt lower bound on $s_{\mathrm{rew}}(x_i)$ in Eq.~\eqref{eq:stage1_reward_score}, applied before mixing with the response-entropy score. \\
Total target zero-variance ratio      & $\alpha$                & $0.25$     & Per-step zero-variance budget that drives the adaptive update of $p_{e,\tau}$. \\
Easy target ratio                     & $\alpha_{\mathrm{easy}}$ & {$\alpha/3{=}0.083$} & Targeted fraction of all-correct zero-variance groups. \\
Hard target ratio                     & $\alpha_{\mathrm{hard}}$ & {$2\alpha/3{=}0.167$} & Targeted fraction of all-incorrect zero-variance groups; the larger hard target retains more prompts that may enter the learnable region as the policy improves. \\
Easy classification threshold         & $r_{\mathrm{easy}}$     & $1.0$      & A reward group is counted as easy zero-variance iff every one of its $G$ rewards equals $1.0$. \\
Hard classification threshold         & $r_{\mathrm{hard}}$     & $0.1$      & A reward group is counted as hard zero-variance iff every one of its $G$ rewards equals $0.1$ (extracted but wrong). \\
History-trace size                    & $|\mathcal{T}_i|$       & unbounded  & Each prompt keeps the full sequence of past visits; only the trailing zero-variance run length $z_i^{(t)}$ enters Eq.~\eqref{eq:stage1_reward_score}. \\
\midrule
\multicolumn{4}{l}{\textit{Stage~2: online entropy verification}} \\
\midrule
Entropy keep ratio                    & $k_{\mathrm{keep}}$     & $0.5$      & Fraction of candidates retained by the entropy gate in Eq.~\eqref{eq:median_gate}.\\%

Entropy upper-trim ratio              & $k_{\mathrm{off}}$      & {$0$}     & Fraction of the highest-entropy candidates.\\%
\midrule
\multicolumn{4}{l}{\textit{Batch sizing}} \\
\midrule
Final training batch (prompts)        & $B_{\mathrm{t}}$        & $256$ / $128$ & $256$ for Qwen2.5-Math-1.5B/7B, Qwen3-4B-Base, and Llama-3.2-3B-Instruct; $128$ for DeepSeek-R1-Distill-Qwen-1.5B. \\
Post-gate rollout target (prompts)    & $B_{\mathrm{cand}}$     & $384$ / $192$ & Equal to $1.5\,B_{\mathrm{t}}$ across these configurations. \\
Group size (responses/prompt)         & $G$                     & $8$        & Number of responses sampled per prompt at rollout. \\
Dataloader micro-batch (prompts)      & $b_{\mathrm{raw}}$      & $32$       & Prompts pulled per inner iteration of the selector loop. \\
Adaptive-resample lower bound         & $b_{\mathrm{min}}$      & $64$$^{\dagger}$ & Floor on the adaptive resample size used by the top-up loop. \\
Adaptive-resample safety factor       & $\eta$                  & $1.25$$^{\dagger}$ & Safety multiplier on the predicted resample size in the top-up loop. \\
\bottomrule
\end{tabular}
\end{table}

\paragraph{Initialization.}
The data profiler storing $\mathcal{T}=\{\mathcal{T}_i\}$ is initialized as an empty dictionary, so a prompt $x_i$ has no history until its first rollout. For an unseen prompt, $z_i^{(t)}{=}0$, hence $s_{\mathrm{rew}}(x_i){=}1$ and Stage~1 accepts it with probability~$1$. Thus, every prompt receives one free visit before history-based filtering applies. We initialize $p_{e,\mathrm{easy}}^{(0)}{=}p_{e,\mathrm{hard}}^{(0)}{=}0.5$; the larger hard target $\alpha_{\mathrm{hard}}{=}2\alpha/3$ then retains hard zero-variance prompts more aggressively. In our main runs, typical steady-state values are $p_{e,\mathrm{easy}}{\approx}0.05{-}0.15$ and $p_{e,\mathrm{hard}}{\approx}0.6{-}0.9$. Stage~2 is stateless and needs no initialization. The selector uses the same global random seed as the verl trainer ($42$ unless specified otherwise), making the Stage 1 Bernoulli draws deterministic given a checkpoint.

\paragraph{Exact batch-sizing protocol.}
A GRPO update consumes exactly $B_{\mathrm{t}}\cdot G$ rollouts. HIVE obtains this fixed-size update batch with the following three-step accumulation procedure, which instantiates Algorithm~\ref{alg:hive-training}.

\begin{enumerate}\itemsep=2pt\parsep=0pt\topsep=2pt
\item \textit{Candidate accumulation.} Pull $b_{\mathrm{raw}}{=}32$ prompts from the dataloader, apply the Stage~1 Bernoulli filter with acceptance probability $P_{\mathrm{S1}}(x_i)$, then {apply the Stage~2 entropy band gate: sort the survivors by $V_t(x)$ in descending order, drop the top $k_{\mathrm{off}}$ fraction, keep the next $k_{\mathrm{keep}}$ fraction, and round the kept count down to a multiple of $G$.} Append the kept prompts to $\mathcal{C}_t$ until $|\mathcal{C}_t|\geq B_{\mathrm{cand}}^{\mathrm{adapt}}$, initialized as $B_{\mathrm{cand}}$.
\item \textit{Rollout filtering.} Sample $G$ responses for every $x\in\mathcal{C}_t$, score them, discard groups whose reward variance is exactly zero, and append the surviving prompt-response groups to $\mathcal{R}_t$.
\item \textit{Top-up and slicing.} Let $M_t=B_t-|\mathcal R_t|/G$ be the number of missing complete prompt groups. When $M_t>0$, update the post-gate rollout target as
\begin{equation}
\label{eq:adaptive-topup}
B_{\mathrm{cand}}^{\mathrm{adapt}}\gets
\begin{cases}
\max\{b_{\min},\min\{B_{\mathrm{cand}},\lceil\eta M_t/(1-\hat\rho_{\mathrm{zv}})\rceil\}\},
&\hat\rho_{\mathrm{zv}}<1,\\
B_{\mathrm{cand}},&\hat\rho_{\mathrm{zv}}=1.
\end{cases}
\end{equation}
Here $\hat\rho_{\mathrm{zv}}$ is measured over the latest rollout round. Repeat candidate accumulation and generation until enough non-zero-variance groups are available. Then keep the first $B_t$ complete groups in arrival order and pass their $B_tG$ responses to GRPO. The slice is taken in arrival order, which is uniform over surviving rollouts because dataloader fetches are independent and
Stage 1/2 acceptance is conditionally independent of buffer position.%
\end{enumerate}

If accumulation terminates, this protocol supplies a $B_t$-prompt update. The number of raw prompts inspected depends on Stage~1 acceptance, the entropy gate and its rounding, and repeated generation caused by zero-variance filtering. %

\section{{Experimental Settings for Preliminary Analysis}}
\label{sec:preliminary-setup}

To derive the empirical observations regarding prompt utility and metadata staleness presented in Section~\ref{sec:obs-preliminary}, we conducted a focused study using the following setup.

\textbf{Model and Dataset.} We used Qwen2.5-Math-1.5B~\cite{yang2024qwen25mathtechnicalreportmathematical} as the base policy model. The model was trained on the DAPO+MATH dataset~\cite{yu2025dapo}. %

\textbf{Training Configuration.} The model was trained using GRPO~\cite{shao2024deepseekmath} with a group size of $G=8$. To capture the diversity required for entropy analysis, we set the sampling temperature to $T=1.0$. Training was conducted on 8$\times$NVIDIA A100 (80GB) GPUs.

\textbf{Metric Definitions.}
\begin{itemize}
    \item \textbf{Effective prompt ratio.} A prompt is effective at a training step if its rollout group has non-zero reward variance and therefore yields a non-vanishing GRPO advantage. The effective prompt ratio is the fraction of selected prompts that are effective under the current policy.
    \item \textbf{Gradient norm (utility).} We quantify the learning signal of a prompt by the $L_2$ norm of its policy gradient with respect to the last layer, normalized by response length. This serves as a proxy for the update magnitude contributed by the prompt.
    \item \textbf{Empirical difficulty.} We estimate prompt difficulty from rollout accuracy. For a prompt $x$ evaluated over a window $\mathcal{W}$, we compute
    \begin{equation}
        A_{\mathcal{W}}(x)=\frac{1}{G|\mathcal{W}|}\sum_{t\in\mathcal{W}}\sum_{i=1}^{G}\mathbb{I}(r_{i,t}=1).
    \end{equation}
    Low accuracy corresponds to hard prompts, while high accuracy corresponds to easy prompts.
    \item \textbf{Response-side entropy.} For each prompt, we compute the average token-level entropy of generated responses and then average over the $G$ rollouts.
\end{itemize}

\textbf{Figure~\ref{fig:intro-history-staleness}: effective-prompt retention.}
We compare a history-only selector with a history-plus-online selector. The history-only selector ranks candidates using cached reward and response-entropy metadata from previous training steps. The history-plus-online selector uses the same historical candidate source but revalidates candidates under the current policy before rollout. We report the effective prompt ratio throughout training.

\textbf{Figure~\ref{fig:intro-high-entropy}: high-entropy non-zero-gradient prompts.}
To isolate utility differences among prompts that already pass zero-variance filtering, we compare DS with a diagnostic DS+High-Entropy variant. DS+High-Entropy prioritizes prompts with higher response-side entropy among prompts with non-zero gradients. We plot evaluation accuracy against the cumulative number of non-zero-gradient rollouts, so the comparison is made under matched informative-rollout budgets.

\textbf{Figure~\ref{fig:intro-learning-edge}: learning edge.}
We bin prompts by empirical difficulty and response-side entropy, using five bins for each axis. For every bin, we report the mean length-normalized gradient norm. The heatmap shows that utility is largest when prompts are neither trivial nor intractable and still induce high response-side uncertainty.

\textbf{Figure~\ref{fig:intro-utility-shift}: moving of the utility.}
We sample 515 non-zero reward variance prompts observed at steps 500 or 1000. Prompts with zero reward variance are assigned to the Zero class. The remaining non-zero-gradient prompts are split into High and Low utility classes by their length-normalized gradient norm at each step. The Sankey diagram reports transitions among these classes, quantifying how historical utility labels become stale as the policy evolves.%

\section{Detailed Related Work}
\label{appendix-sec:detailed_related_work}
\textbf{RL for LLM reasoning.}
Reinforcement learning has become a key technique for post-training LLMs~\cite{christiano2017deep,bai2022training}. Early RLHF work~\cite{ouyang2022training,rafailov2023direct} used human preference signals to improve instruction following and alignment. More recent RLVR systems replace learned reward models with verifiable task rewards, producing large gains in mathematics and programming~\cite{guo2025deepseek,hu2025reinforce++,liu2503understanding,jaech2024openai,Zhang2025SRPOAC}. Building on RLVR, algorithmic work has improved PPO-style and GRPO-style training. VC-PPO~\cite{yuan2025s} and VAPO~\cite{yue2025vapo} strengthen value-function learning, while RLOO~\cite{ahmadian2024back} and GRPO~\cite{shao2024deepseekmath} avoid explicit value learning with multi-sample baselines. Other studies introduce replay mechanisms~\cite{liang2025squeeze,li2025repo,zhang2025rlep,lu2025safe} or adjust optimization objectives to improve stability and reduce bias~\cite{liu2503understanding,chu2025gpg,huang2025mapo,zheng2025group,www2025TFDCon,yyt2024eccv}. These works are complementary to HIVE: they improve how selected rollouts update the policy, whereas HIVE improves which prompts receive rollouts before the update.

\textbf{Data-efficient LLM training.}
Prompt and trajectory quality strongly affect LLM training efficiency. Offline curation follows the ``less is more'' principle from supervised fine-tuning~\cite{zhou2023lima,ye2025limo,zheng2024picture,wang2025comprehensivesurveyllmagentstack,tkde2024DConRec} and reinforced fine-tuning~\cite{Wang2025ReinforcementLF,Fatemi2025ConciseRV,Li2025LIMRLI,shi2025efficient,tang2025towards,cikm2022DcRec}. These approaches can reduce dataset size before training, but they cannot respond to the model's changing competence during RLVR.

Online selection addresses this limitation by conditioning selection on the current policy. Rollout-based methods discard uninformative prompts via real-time sampling~\cite{yu2025dapo,meng2025mm,foster2025learning,xu2025not,lin2025cppo,sun2025efficient} or prioritize medium-difficulty prompts~\cite{bae2025online}. They are online and precise, but they require repeated response generation before filtering. Estimator-based methods instead maintain historical logs~\cite{zheng2025act} or Bayesian utility estimates~\cite{chen2025self,zeng2025cures,shen2025bots}. These methods are efficient, but stale metadata can misrepresent the current policy's learning edge. Auxiliary-model approaches use external scorers or jointly trained predictors to estimate difficulty, as in DOTS~\cite{sun2025dots} and PCL~\cite{Gao2025PCL}; these improve adaptivity but add extra computation or memory. Stage-wise curriculum refreshes~\cite{zhang2025learning,Zhang2025SRPOAC} periodically re-estimate difficulty, but they still do not provide cheap per-step online verification.

HIVE is motivated by the observation that none of these categories simultaneously satisfy the online, and efficient requirements. Its history-informed stage inherits the low overhead of historical filtering, while its prompt-side entropy verifier corrects stale metadata without generating responses.

\section{Additional Experiments}
\label{appendix-sec:additional-exp}
\subsection{Correlation Between Prompt Entropy and Response Entropy}
\label{appendix-subsec:entropy-robustness}

Building on the empirical validation of assumption in Section~\ref{appendix-sec:proof-thm}, we conduct an additional robustness study by redefining the response entropy $S_{\mathrm{resp}}(x)$ as the entropy computed over only the top-$r\%$ tokens of the response distribution, sweeping $r \in \{10,20,30,40\}$. To suppress intrinsic decoding stochasticity, we follow the same binned-mean protocol on 2{,}048 prompts, binning by $S_{\mathrm{prompt}}(x)$ (proxied by $V(x)$) and averaging $S_{\mathrm{resp}}^{(r)}(x)$ within each bin. As shown in Figure~\ref{fig:combined_results_appendix}, the prompt--response relationship remains linear across all ratios.
This indicates that the prompt--response entropy relationship is robust to this response-entropy definition, further supporting the bounded-noise linear approximation used in Section~\ref{appendix-sec:proof-thm}.

\begin{figure}[!htbp]
    \centering
    \begin{subfigure}[b]{0.45\linewidth}
        \centering
        \includegraphics[width=\linewidth]{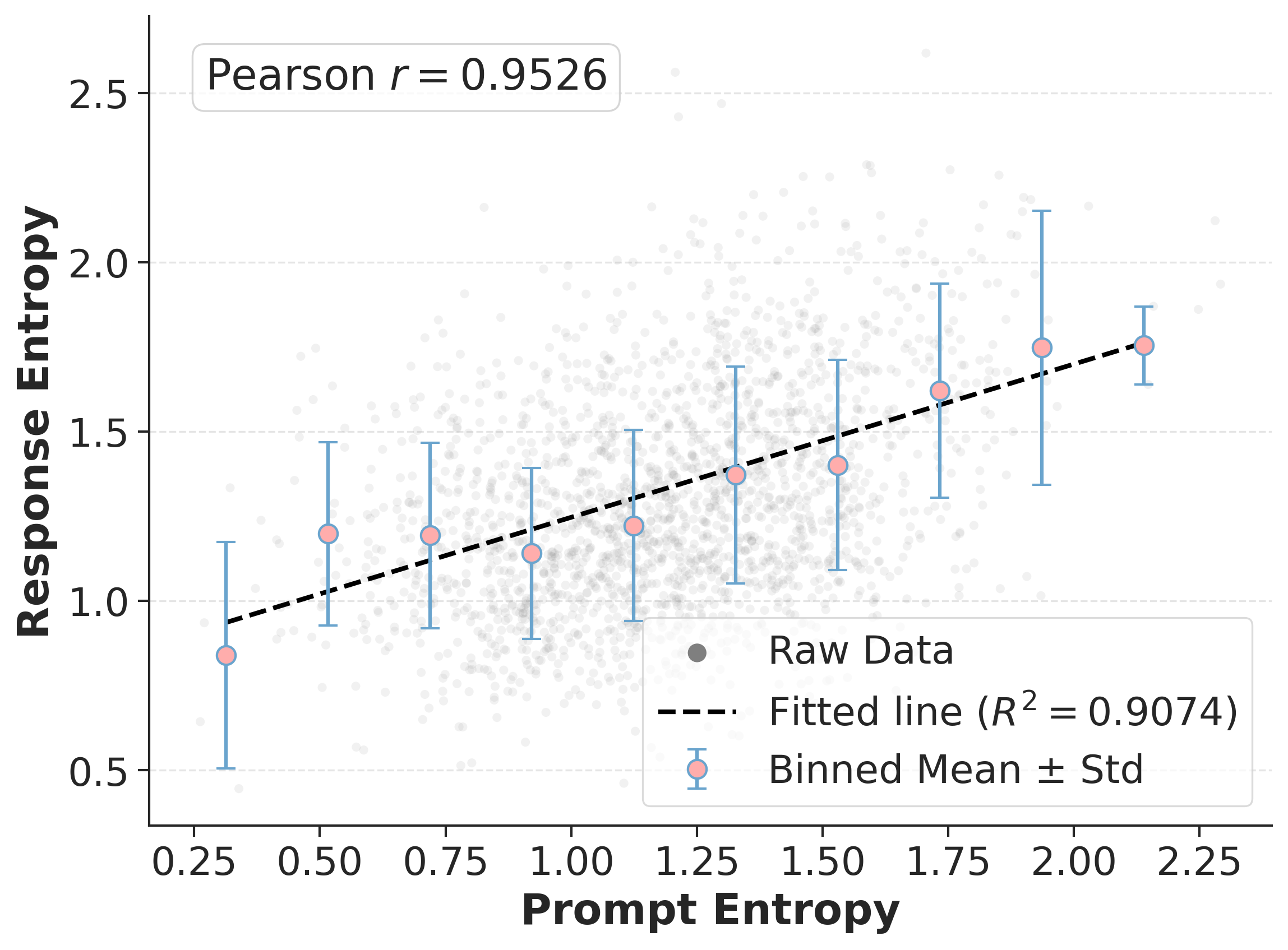}
        \caption{$r=10\%$}
        \label{fig:respRatio10}
    \end{subfigure}
    \hfill
    \begin{subfigure}[b]{0.45\linewidth}
        \centering
        \includegraphics[width=\linewidth]{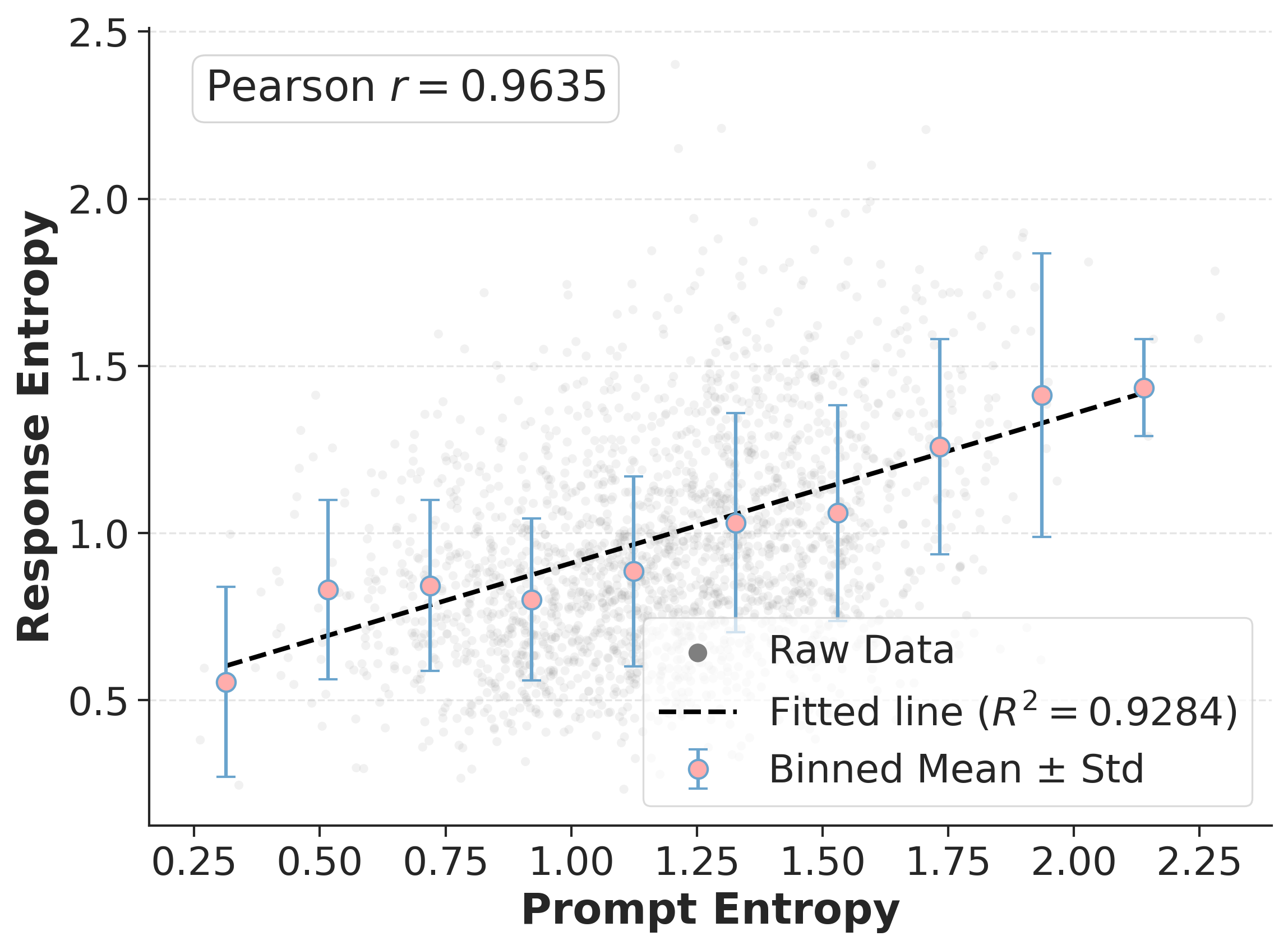}
        \caption{$r=20\%$}
        \label{fig:respRatio20}
    \end{subfigure}

    \vspace{0.4em} %

    \begin{subfigure}[b]{0.45\linewidth}
        \centering
        \includegraphics[width=\linewidth]{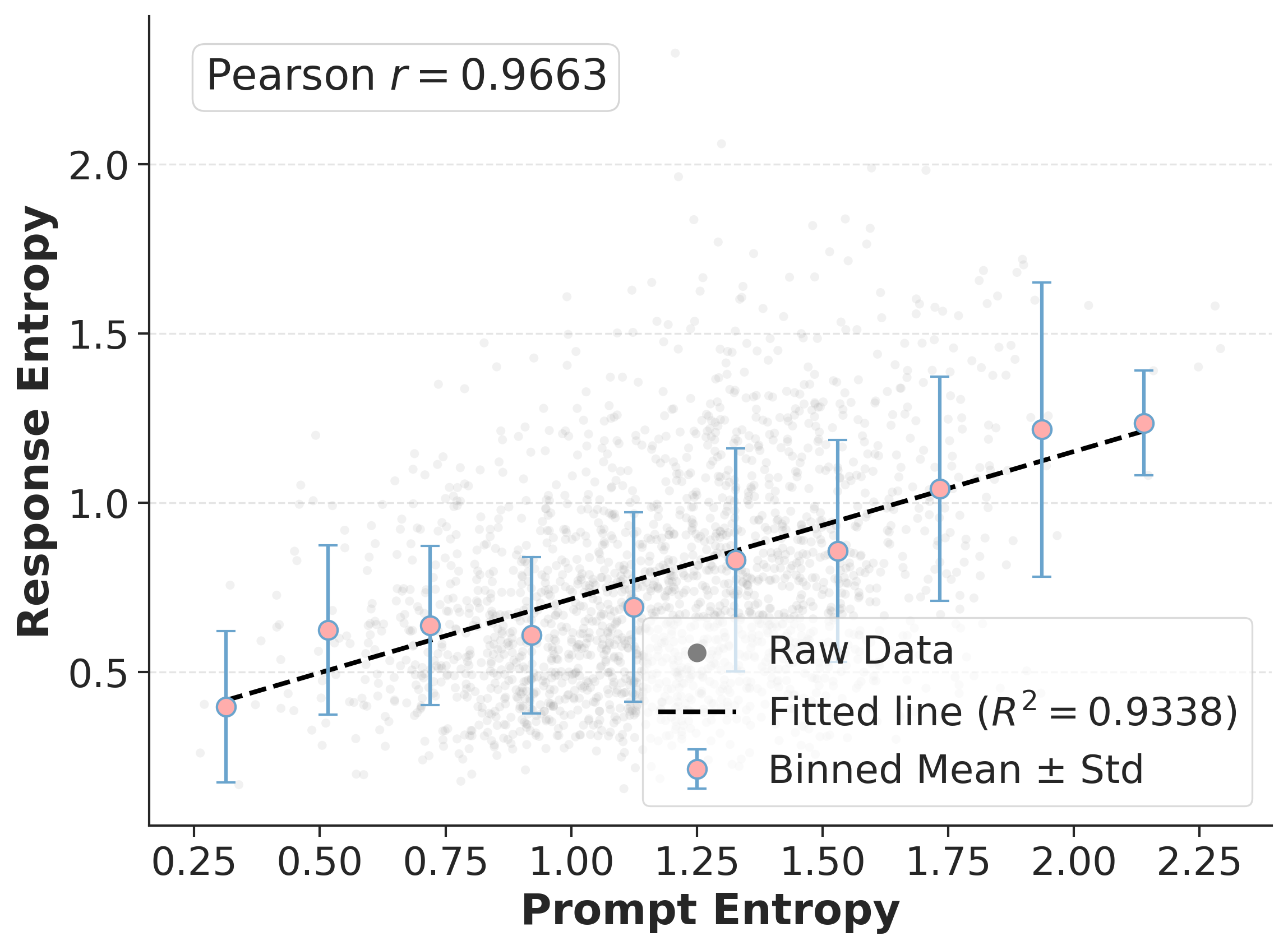}
        \caption{$r=30\%$}
        \label{fig:respRatio30}
    \end{subfigure}
    \hfill
    \begin{subfigure}[b]{0.45\linewidth}
        \centering
        \includegraphics[width=\linewidth]{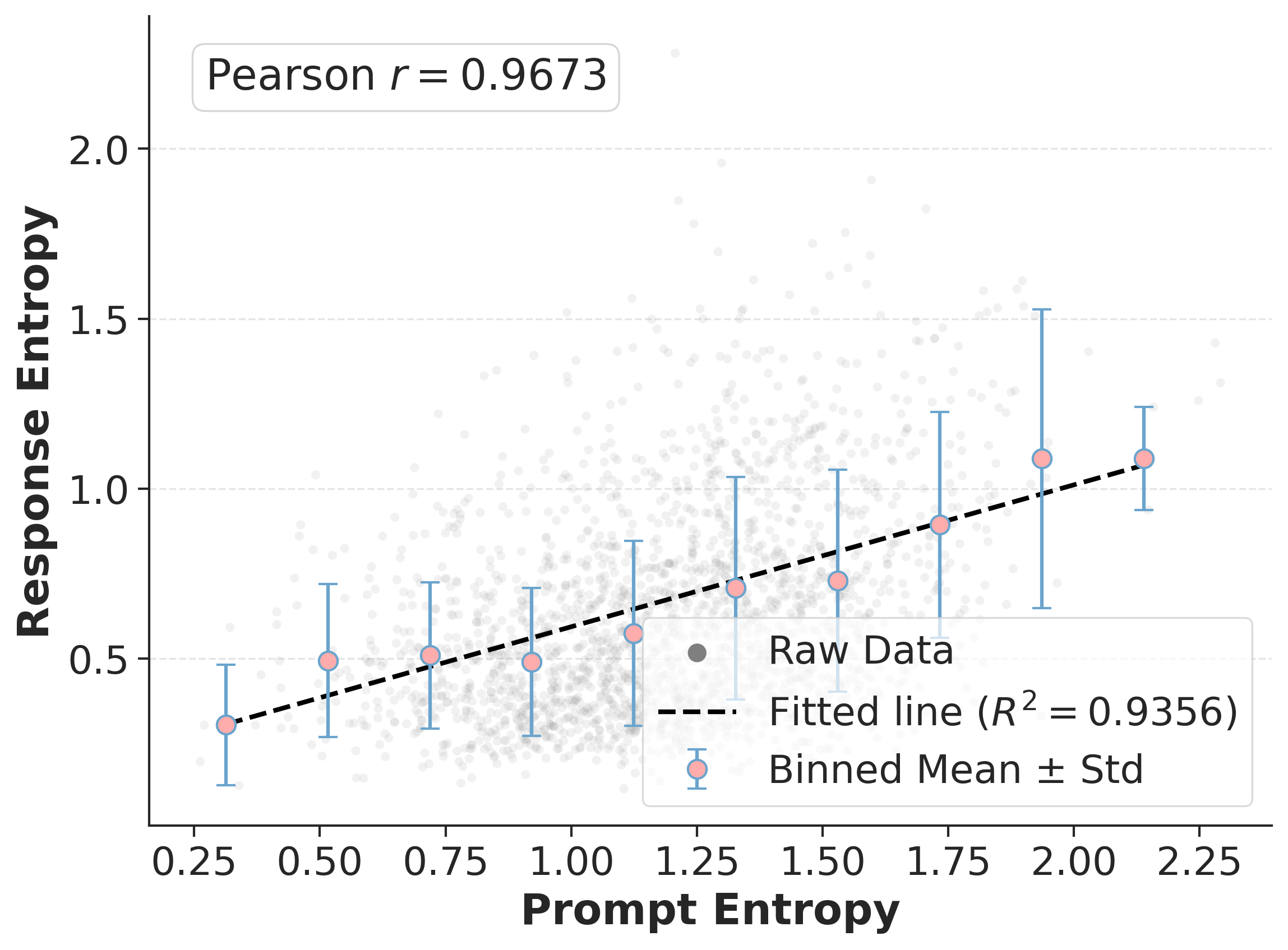}
        \caption{$r=40\%$}
        \label{fig:respRatio40}
    \end{subfigure}

    \caption{Relationship between prompt entropy and response entropy computed over the top-$r\%$ tokens in the response distribution, sweeping $r\in\{10,20,30,40\}$. Each panel reports the binned mean trend under the corresponding ratio setting.}
    \label{fig:combined_results_appendix}
\end{figure}

\subsection{Surface-Feature Controls for Stage 2}
\label{appendix-subsec:surface-controls}

Table~\ref{tab:surface-feature-controls} compares alternative Stage~2 scores on Qwen2.5-Math-1.5B trained on DAPO+MATH. Each control keeps Stage~1 and the GRPO configuration fixed and replaces only the score used by the entropy gate. We report one run per control. The four scores are defined as follows:
\begin{itemize}[leftmargin=*,itemsep=2pt,topsep=3pt]
    \item \textbf{Prompt length:} the attention-masked token count actually seen by the model after applying the chat template and left truncation to a maximum of 1,536 tokens.
    \item \textbf{Rare-token ratio:} the fraction of content tokens whose token types belong to the rarest types covering the bottom 10\% of training-corpus token mass. The realized cutoff is a corpus count of at most 637.
    \item \textbf{Math-symbol density:} the number of characters in a predefined set of mathematical, Greek, and set-notation symbols, together with characters in LaTeX commands matching \texttt{\textbackslash\textbackslash[A-Za-z]+}, divided by the decoded character count.
    \item \textbf{Formatting/tokenization:} the equal-weight mean of six features after batch-wise $z$-score normalization: tokens-per-word fragmentation and the densities of punctuation, whitespace, newlines, non-ASCII characters, and code fences.
\end{itemize}
These controls use surface statistics without a current-policy forward pass. Input perplexity is a distinct online alternative that requires policy scoring; it is not evaluated in this comparison.

\subsection{Case Study}

To gain deeper insight into the behavior of our selective filtering algorithm, we analyze a case study based on prompts from the MATH dataset.
We divide these prompts into four categories: Frequently Skipped Prompts (Easy), Frequently Skipped Prompts (Hard), Frequently Selected Prompts, and Prompts Frequently Selected by Stage 1 but Skipped by Stage 2.
We observe that frequently skipped easy prompts often involve simple computations or routine use of formulas, leading to high solution rates across sampled responses. In contrast, hard prompts that are often skipped tend to be too difficult for the model, leading to low or zero success across rollouts, limiting their value for training.
Frequently selected prompts tend to be moderately difficult, making consistent contributions to model learning.
Furthermore, we analyze the prompts selected in stage 1 but filtered after stage 2. Most
prompts are relatively easy, showing that the historical information is not reliable for filtering.

\begin{tcolorbox}[casestudy={Frequently Skipped Prompts (Easy)}]
1. \textbf{Question:} What is the value of $(2x + 5)^2$ when $x = 3$? \textbf{Solution:} 121.
\vspace{0.2cm}

2. \textbf{Question:} Solve for $x$: $5^{x + 4} = 125^x$?
\textbf{Solution:} 2.
\vspace{0.2cm}

3. \textbf{Question:} A squirrel travels at a constant 4 miles per hour. How long does it take for this squirrel to travel 1 mile? Express your answer in minutes.
\textbf{Solution:} 15.
\vspace{0.2cm}

4. \textbf{Question:} Evaluate $\text{log}_5 625$.
\textbf{Solution:} 4.
\vspace{0.2cm}

5. \textbf{Question:} What value of $x$ will give the minimum value for $x^2- 10x + 24$?
\textbf{Solution:} 5.
\end{tcolorbox}

\begin{tcolorbox}[casestudy={Frequently Skipped Prompts (Hard)}]
1. \textbf{Question:} Let $\triangle{PQR}$ be a right triangle with $PQ = 90$, $PR = 120$, and $QR = 150$. Let $C_{1}$ be the inscribed circle. Construct $\overline{ST}$ with $S$ on $\overline{PR}$ and $T$ on $\overline{QR}$, such that $\overline{ST}$ is perpendicular to $\overline{PR}$ and tangent to $C_{1}$. Construct $\overline{UV}$ with $U$ on $\overline{PQ}$ and $V$ on $\overline{QR}$ such that $\overline{UV}$ is perpendicular to $\overline{PQ}$ and tangent to $C_{1}$. Let $C_{2}$ be the inscribed circle of $\triangle{RST}$ and $C_{3}$ the inscribed circle of $\triangle{QUV}$. The distance between the centers of $C_{2}$ and $C_{3}$ can be written as $\sqrt {10n}$. What is $n$?
\textbf{Solution:} 725.
\vspace{0.2cm}

2. \textbf{Question:} Consider the region $A^{}_{}$ in the complex plane that consists of all points $z^{}_{}$ such that both $\frac{z^{}_{}}{40}$ and $\frac{40^{}_{}}{\overline{z}}$ have real and imaginary parts between $0^{}_{}$ and $1^{}_{}$, inclusive. Find the area of $A.$
\textbf{Solution:} $1200 - 200 \pi$.
\vspace{0.2cm}

3. \textbf{Question:} Ephram is growing $3$ different variants of radishes in a row of $13$ radishes total, but he forgot where he planted each radish variant and he can't tell what variant a radish is before he picks it. Ephram knows that he planted at least one of each radish variant, and all radishes of one variant will form a consecutive string, with all such possibilities having an equal chance of occurring. He wants to pick three radishes to bring to the farmers market, and wants them to all be of different variants. Given that he uses optimal strategy, the probability that he achieves this can be expressed as $\frac{m}{n}$, where $m$ and $n$ are relatively prime positive integers. Find $m+n$.
\textbf{Solution:} 17.
\vspace{0.2cm}

4. \textbf{Question:} There are 1000 rooms in a row along a long corridor. Initially, the first room contains 1000 people, and the remaining rooms are empty. Each minute, the following happens: for each room containing more than one person, someone in that room decides it is too crowded and moves to the next room. All these movements are simultaneous (so nobody moves more than once within a minute). After one hour, how many different rooms will have people in them?
\textbf{Solution:} 31.
\vspace{0.2cm}

5. \textbf{Question:} A regular hexagon has side length \(6\). Congruent arcs with radius \(3\) are drawn with the center at each of the vertices, creating circular sectors as shown. The region inside the hexagon but outside the sectors is shaded as shown. The original answer is in the form \(a\sqrt{b}-c\pi\). Please find the value of a + b + c.
\textbf{Solution:} 75.
\end{tcolorbox}

\begin{tcolorbox}[casestudy={Frequently Selected Prompts After Stage 2}]
1. \textbf{Question:} Let $z$ be a complex number such that
\[|z - 12| + |z - 5i| = 13.\]
Find the smallest possible value of $|z|.$ \textbf{Solution:} $\frac{60}{13}$.
\vspace{0.2cm}

2. \textbf{Question:} Find all real solutions to $x^4+(2-x)^4=34$.  Enter all the solutions, separated by commas.
\textbf{Solution:} $1 + \sqrt{2}, 1 - \sqrt{2}$.
\vspace{0.2cm}

3. \textbf{Question:} Let b be a real number randomly selected from the interval $[-17,17]$. Then, m and n are two relatively prime positive integers such that m\/n is the probability that the equation $x^4+25b^2=(4b^2-10b)x^2$ has $\textit{at least}$ two distinct real solutions. Find the value of $m+n$.
\textbf{Solution:} 63.
\vspace{0.2cm}

4. \textbf{Question:} Let $a,$ $b,$ $c,$ $d$ be real numbers such that\[\frac{(a - b)(c - d)}{(b - c)(d - a)} = \frac{2}{5}.\]Find the sum of all possible values of\[\frac{(a - c)(b - d)}{(a - b)(c - d)}.\]
\textbf{Solution:} $-\frac{3}{2}$.
\vspace{0.2cm}

5. \textbf{Question:} The function $f$, defined on the set of ordered pairs of positive integers, satisfies the following properties:
\begin{align*}  f(x,x) &=x, \\  f(x,y) &=f(y,x), \quad \text{and} \\ (x + y) f(x,y) &= yf(x,x + y).
\end{align*}
Calculate $f(14,52)$.
\textbf{Solution:} 364.
\end{tcolorbox}

\begin{tcolorbox}[casestudy={Prompts Frequently Selected by Stage 1 but Skipped by Stage 2}]
1. \textbf{Question:} A polynomial with integer coefficients is of the form
\[9x^4 + a_3 x^3 + a_2 x^2 + a_1 x + 15 = 0.\]
Find the number of different possible rational roots of this polynomial.
\textbf{Solution:} 16.
\vspace{0.2cm}

2. \textbf{Question:} When an integer is divided by 15, the remainder is 7. Find the sum of the remainders when the same integer is divided by 3 and by 5.
\textbf{Solution:} 3.
\vspace{0.2cm}

3. \textbf{Question:} Find the number of positive divisors of 2002.
\textbf{Solution:} 16.
\vspace{0.2cm}

4. \textbf{Question:} What is the base five product of the numbers $121_{5}$ and $11_{5}$?
\textbf{Solution:} 1331.
\vspace{0.2cm}

5. \textbf{Question:} If $\sin x = 3 \cos x,$ then what is $\sin x \cos x$
\textbf{Solution:} $\frac{3}{10}$.
\end{tcolorbox}

\section{{Algorithm}}
\label{appendix-sec:alg}

Algorithm~\ref{alg:hive-training} expands the fixed-size training-batch procedure in Appendix~\ref{appendix-subsec:hive-hyperparams}. The target $B_{\mathrm{cand}}=1.5B_t$ counts prompts \emph{after} both selection stages. Within each raw micro-batch, Stage~1 performs Bernoulli selection using the historical acceptance probability $P_{\mathrm{S1}}$ in Eq.~\eqref{eq:stage1_select}. Stage~2 ranks the survivors by current prompt entropy, applies the keep rule and rounding specified in Appendix~\ref{appendix-subsec:hive-hyperparams}, and appends the retained prompts to $\mathcal{B}_t$. Candidate accumulation continues until the current post-gate target is reached.

\textbf{Zero-variance filtering and top-up.} Each generation round produces $G$ responses per selected prompt. A group is admitted to the response buffer $\mathcal{R}_t$ only when its reward variance is nonzero; groups remain intact and retain their arrival order. If fewer than $B_t$ complete groups have accumulated, HIVE estimates the zero-variance ratio from the latest generation round, updates the next post-gate target using Eq.~\eqref{eq:adaptive-topup}, and repeats selection and generation with the same policy. When the latest round is entirely zero-variance, the next target is reset to $B_{\mathrm{cand}}$, avoiding division by zero.

\textbf{Slicing and policy update.} Once enough groups are available, HIVE takes the first $B_t$ complete groups in arrival order and performs one GRPO update on their $B_tG$ responses. Excess groups are not used in this update. After the policy update, the history and exploration statistics are refreshed using $\mathcal{A}_t$, which records every generated group across all rounds, including zero-variance and excess groups. All generated responses count toward rollout cost, regardless of whether they enter $\mathcal{R}_{\mathrm{train}}$.

A complete list of every selector hyperparameter (including $\lambda$, $\alpha$, $\Delta p$, $\epsilon_p$, $p_{\min}/p_{\max}$, $k_{\mathrm{keep}}$, $k_{\mathrm{off}}$, the trace and unseen-prompt initialization, and the exact batch-sizing protocol that produces a $B_{\mathrm{t}}$-prompt GRPO update) is provided in Appendix~\ref{appendix-subsec:hive-hyperparams}.

\begin{algorithm}[H]
\caption{Training Iteration in HIVE}
\label{alg:hive-training}
\begin{algorithmic}[1]
    \REQUIRE Dataset $\mathcal{D}$; policy $\pi_\theta$; history $\mathcal{T}$; training batch size $B_t$ (prompts); group size $G$; post-gate rollout target $B_{\mathrm{cand}}=1.5B_t$; raw micro-batch size $b_{\mathrm{raw}}$; resampling parameters $b_{\min},\eta$; selector settings in Appendix~\ref{appendix-subsec:hive-hyperparams}.
    \STATE $\mathcal{R}_t\gets[\,]$ (response buffer); $\mathcal{A}_t\gets[\,]$ (all generated groups)
    \STATE $B_{\mathrm{cand}}^{\mathrm{adapt}}\gets B_{\mathrm{cand}}$
    \REPEAT
        \STATE $\mathcal{B}_t\gets[\,]$ \quad\textcolor{gray}{\textit{\# Fresh post-gate prompt buffer}}
        \REPEAT
            \STATE Draw a raw micro-batch $\mathcal{X}$ of $b_{\mathrm{raw}}$ prompts from $\mathcal{D}$
            \STATE $\mathcal{S}\gets\{x_i\in\mathcal{X}:\mathcal{T}_i^{(t)}=\varnothing\}$ \COMMENT{Unseen: bypass Stage~1}
            \STATE $\mathcal{S}\gets\mathcal{S}\cup\{x_i\in\mathcal{X}:\mathcal{T}_i^{(t)}\neq\varnothing,\ \operatorname{Bernoulli}(P_{\mathrm{S1}}(x_i))=1\}$
            \STATE Compute $V_t(x)$ for $x\in\mathcal{S}$ and sort by decreasing entropy
            \STATE Retain the top $G\lfloor k_{\mathrm{keep}}|\mathcal{S}|/G\rfloor$ prompts ($k_{\mathrm{off}}=0$)
            \STATE Append the retained prompts to $\mathcal{B}_t$
        \UNTIL{$|\mathcal{B}_t|\geq B_{\mathrm{cand}}^{\mathrm{adapt}}$}
        \STATE \textcolor{gray}{\textit{\# Generate and filter complete prompt-response groups}}
        \STATE $\mathcal{W}\gets$ generate $G$ responses per prompt in $\mathcal{B}_t$; score rewards and entropies
        \STATE Append all groups in $\mathcal{W}$ to $\mathcal{A}_t$
        \STATE $\mathcal{Z}\gets\{w\in\mathcal{W}:\operatorname{Var}(\operatorname{rewards}(w))=0\}$
        \STATE Append responses from $\mathcal{W}\setminus\mathcal{Z}$ to $\mathcal{R}_t$ in arrival order, keeping groups intact
        \STATE $\hat\rho_{\mathrm{zv}}\gets |\mathcal{Z}|/|\mathcal{W}|$; $M_t\gets B_t-|\mathcal{R}_t|/G$
        \IF{$M_t>0$}
            \IF{$\hat\rho_{\mathrm{zv}}<1$}
                \STATE $B_{\mathrm{cand}}^{\mathrm{adapt}}\gets\max\{b_{\min},\min\{B_{\mathrm{cand}},\lceil\eta M_t/(1-\hat\rho_{\mathrm{zv}})\rceil\}\}$
            \ELSE
                \STATE $B_{\mathrm{cand}}^{\mathrm{adapt}}\gets B_{\mathrm{cand}}$
            \ENDIF
        \ENDIF
    \UNTIL{$|\mathcal{R}_t|\geq B_tG$}
    \STATE $\mathcal{R}_{\mathrm{train}}\gets$ first $B_t$ complete groups in $\mathcal{R}_t$; discard excess from training
    \STATE Update $\pi_\theta$ with GRPO on the $B_tG$ responses in $\mathcal{R}_{\mathrm{train}}$
    \STATE Update $\mathcal{T}$ with rewards and response entropies from all groups in $\mathcal{A}_t$
    \STATE $n_{\mathrm{total}}\gets|\mathcal{A}_t|$; count $n_{\mathrm{easy}},n_{\mathrm{hard}}$ over all groups in $\mathcal{A}_t$
    \FOR{$\tau\in\{\mathrm{easy},\mathrm{hard}\}$}
        \IF{$n_\tau/n_{\mathrm{total}}>\alpha$}
            \STATE $p_{e,\tau}\gets p_{e,\tau}-\Delta p$
        \ELSIF{$n_\tau/n_{\mathrm{total}}<\alpha$}
            \STATE $p_{e,\tau}\gets p_{e,\tau}+\Delta p$
        \ELSE
            \STATE Keep $p_{e,\tau}$ unchanged \COMMENT{$n_\tau/n_{\mathrm{total}}=\alpha$}
        \ENDIF
        \STATE $p_{e,\tau}\gets\min\{p_{\max},\max\{p_{\min},p_{e,\tau}\}\}$
    \ENDFOR

\end{algorithmic}
\end{algorithm}

\clearpage
\section{Asset Licenses}
\label{sec:asset-licenses}
Table~\ref{tab:assets} identifies the model checkpoints, data distributions, and principal software assets used in this study, distinguishing declared licenses from unresolved documentation.

\paragraph{Open-R1 subset and version.}
Our 30K subset is drawn from \href{https://huggingface.co/datasets/open-r1/OpenR1-Math-220k}{\texttt{open-r1/OpenR1-Math-220k}}, configuration \texttt{default}, split \texttt{train}: the preprocessing script selects rows 0--29,999 in source order, without shuffling, and uses the \texttt{problem} and \texttt{answer} fields. The original Hugging Face revision was not recorded; a fresh download therefore does not establish an identical ordered subset.

\begin{table}[!h]
\centering
\caption{Asset sources and declared licenses. HF denotes Hugging Face and GH denotes GitHub; source names and version identifiers link to the corresponding assets.}
\label{tab:assets}
\begingroup
\small
\setlength{\tabcolsep}{3.5pt}
\renewcommand{\arraystretch}{1.04}
\begin{tabular}{@{}>{\raggedright\arraybackslash}p{0.32\linewidth}>{\raggedright\arraybackslash}p{0.43\linewidth}>{\raggedright\arraybackslash}p{\dimexpr0.25\linewidth-14pt\relax}@{}}
\toprule
\textbf{Asset / version} & \textbf{Distribution or source} & \textbf{License} \\
\midrule
\multicolumn{3}{@{}l}{\textit{Models}} \\
Qwen2.5-Math-1.5B/7B & HF: Qwen (\href{https://huggingface.co/Qwen/Qwen2.5-Math-1.5B}{1.5B}, \href{https://huggingface.co/Qwen/Qwen2.5-Math-7B}{7B}) & Apache-2.0 \\
Qwen2.5-1.5B/14B/32B & HF: Qwen (\href{https://huggingface.co/Qwen/Qwen2.5-1.5B}{1.5B}, \href{https://huggingface.co/Qwen/Qwen2.5-14B}{14B}, \href{https://huggingface.co/Qwen/Qwen2.5-32B}{32B}) & Apache-2.0 \\
Qwen3-4B-Base & HF: \href{https://huggingface.co/Qwen/Qwen3-4B-Base}{Qwen/Qwen3-4B-Base} & Apache-2.0 \\
Llama-3.2-3B-Instruct & HF: \href{https://huggingface.co/meta-llama/Llama-3.2-3B-Instruct}{meta-llama/Llama-3.2-3B-Instruct} & \href{https://github.com/meta-llama/llama-models/blob/main/models/llama3_2/LICENSE}{Llama 3.2 Community}; \href{https://github.com/meta-llama/llama-models/blob/main/models/llama3_2/USE_POLICY.md}{AUP} \\
DeepSeek-R1-Distill-\newline Qwen-1.5B & HF: \href{https://huggingface.co/deepseek-ai/DeepSeek-R1-Distill-Qwen-1.5B}{deepseek-ai/DeepSeek-R1-Distill-Qwen-1.5B} & MIT; base model Apache-2.0 \\
\midrule
\multicolumn{3}{@{}l}{\textit{Training and evaluation data}} \\
DAPO-Math-17K-cleaned & HF: \href{https://huggingface.co/datasets/haizhongzheng/DAPO-Math-17K-cleaned}{haizhongzheng/DAPO-Math-17K-cleaned}; train & \href{https://huggingface.co/datasets/BytedTsinghua-SIA/DAPO-Math-17k}{Apache-2.0}$^{\dagger}$ \\
MATH (7,500 training rows) & HF: \href{https://huggingface.co/datasets/xDAN2099/lighteval-MATH}{xDAN2099/lighteval-MATH}; train & \href{https://github.com/hendrycks/math/blob/main/LICENSE}{MIT}$^{\dagger}$ \\
Open-R1 (30K; see above) & HF: \href{https://huggingface.co/datasets/open-r1/OpenR1-Math-220k}{open-r1/OpenR1-Math-220k}; default/train & Apache-2.0 \\
MATH-500 & HF: \href{https://huggingface.co/datasets/HuggingFaceH4/MATH-500}{HuggingFaceH4/MATH-500}; test & \href{https://github.com/openai/prm800k/blob/main/LICENSE}{MIT}$^{\dagger}$ \\
AIME 2024 & HF: \href{https://huggingface.co/datasets/HuggingFaceH4/aime_2024}{HuggingFaceH4/aime\_2024}; train & \href{https://huggingface.co/datasets/AI-MO/aimo-validation-aime}{Apache-2.0}$^{\dagger}$ \\
AMC (AIMO validation) & HF: \href{https://huggingface.co/datasets/AI-MO/aimo-validation-amc}{AI-MO/aimo-validation-amc}; train & Apache-2.0 \\
Gaokao 2023 (English) & \href{https://github.com/Infini-AI-Lab/GRESO/blob/b27784315479921ea1977178d7e6c4858b2ae5a1/raw-data/gaokao-2023en-test.jsonl}{GRESO raw-data/gaokao-2023en-test.jsonl}, \texttt{b277843} & Not stated \\
Minerva Math & HF: \href{https://huggingface.co/datasets/math-ai/minervamath}{math-ai/minervamath}; test & Not stated \\
Olympiad Bench & HF: \href{https://huggingface.co/datasets/math-ai/olympiadbench}{math-ai/olympiadbench}; test & \href{https://huggingface.co/datasets/Hothan/OlympiadBench}{Apache-2.0}$^{\dagger}$ \\
BFCL-V2 (80/20 split) & HF: \href{https://huggingface.co/datasets/gorilla-llm/Berkeley-Function-Calling-Leaderboard/tree/1726e89933a8001f196acd30a2e8e92254f75c5c}{gorilla-llm/BFCL, \texttt{1726e899}}; GH: \href{https://github.com/ShishirPatil/gorilla/tree/30124c4/berkeley-function-call-leaderboard}{Gorilla evaluator, \texttt{30124c4}} & Apache-2.0 (data and evaluator) \\
$4\times4$ Sudoku & HF: \href{https://huggingface.co/datasets/zzy1123/sudoku_dedup}{zzy1123/sudoku\_dedup}; train/test & Not stated \\
\midrule
\multicolumn{3}{@{}l}{\textit{Software and reward design}} \\
Reasoning Gym (reward adaptation) & GH: \href{https://github.com/open-thought/reasoning-gym}{open-thought/reasoning-gym}; revision not recorded & \href{https://github.com/open-thought/reasoning-gym/blob/main/LICENSE}{Apache-2.0} \\
GRESO (training code) & GH: \href{https://github.com/Infini-AI-Lab/GRESO}{Infini-AI-Lab/GRESO} & See note below \\
verl & GH: \href{https://github.com/volcengine/verl}{volcengine/verl}; local fork & \href{https://github.com/volcengine/verl/blob/main/LICENSE}{Apache-2.0} \\
vLLM & GH: \href{https://github.com/vllm-project/vllm}{vllm-project/vllm} & \href{https://github.com/vllm-project/vllm/blob/main/LICENSE}{Apache-2.0} \\
\bottomrule
\end{tabular}
\par\smallskip
\begin{minipage}{\linewidth}
\small
$^{\dagger}$License declared by the linked upstream release or repository; the exact HF mirror does not separately state a license. ``Not stated'' denotes missing distribution-level license documentation, not unrestricted reuse. GRESO retains Apache-2.0 notices in inherited verl files, but its package metadata points to a repository-wide license file that is absent from the inspected distribution. These gaps preclude a blanket claim that licenses and terms are documented for all assets. Sudoku data and the Reasoning Gym reward adaptation are separate assets.
\end{minipage}
\endgroup
\end{table}

\clearpage
\endgroup

\clearpage
\section*{NeurIPS Paper Checklist}

\begin{enumerate}

\begin{samepage}
\item {\bf Claims}
    \item[] Question: Do the main claims made in the abstract and introduction accurately reflect the paper's contributions and scope?
    \item[] Answer: \answerYes{}.
    \item[] Justification: The abstract and Section~\ref{sec:introduction} state the moving-edge tracking claim and the edge-aware, online, and efficient selection objectives. These claims are supported by the methodology in Section~\ref{sec:method}, experiments in Section~\ref{sec:experiments}, and the limitations discussion.
    \end{samepage}

\begin{samepage}
\item {\bf Limitations}
    \item[] Question: Does the paper discuss the limitations of the work performed by the authors?
    \item[] Answer: \answerYes{}.
    \item[] Justification: Section~\ref{sec:limitations} discusses the conservative nature of the conditional rank-consistency analysis, its potentially loose worst-case bounds, and the need for tighter variance-dependent guarantees.
    \end{samepage}

\begin{samepage}
\item {\bf Theory assumptions and proofs}
    \item[] Question: For each theoretical result, does the paper provide the full set of assumptions and a complete (and correct) proof?
    \item[] Answer: \answerYes{}.
    \item[] Justification: {Section~\ref{subsec:stage2} states a conditional pairwise ordering result. Appendix~\ref{appendix-sec:proof-thm} gives the explicit calibration assumption, finite-rollout concentration bound, and proof.}
    \end{samepage}

    \begin{samepage}
\item {\bf Experimental result reproducibility}
    \item[] Question: Does the paper fully disclose all the information needed to reproduce the main experimental results of the paper to the extent that it affects the main claims and/or conclusions of the paper (regardless of whether the code and data are provided or not)?
    \item[] Answer: \answerYes{}.
    \item[] Justification: Section~\ref{sec:method} and Algorithm~\ref{alg:hive-training} specify HIVE. Section~\ref{sec:experiments} and Appendix~\ref{appendix-sec:detailed-experimental-setting} disclose the main training protocol, selector hyperparameters, initialization and random seed, budget-matched baseline comparisons, and checkpoint-selection rule.
    \end{samepage}

\begin{samepage}
\item {\bf Open access to data and code}
    \item[] Question: Does the paper provide open access to the data and code, with sufficient instructions to faithfully reproduce the main experimental results, as described in supplemental material?
    \item[] Answer: \answerYes{}.
    \item[] Justification: The \href{https://github.com/JiahaoWuGit/HIVE-Efficient-Rollout-RL}{public HIVE repository} provides installation instructions, data-preparation scripts, and training/evaluation commands for the main math study. The release focuses on HIVE; baseline, 14B/32B scaling, tool-calling, Sudoku, and ablation pipelines are not bundled. The data guide notes that historical dataset revisions were not archived, and the installation guide explains the separate runtime requirements for Qwen3.
    \end{samepage}

\begin{samepage}
\item {\bf Experimental setting/details}
    \item[] Question: Does the paper specify all the training and test details (e.g., data splits, hyperparameters, how they were chosen, type of optimizer) necessary to understand the results?
    \item[] Answer: \answerYes{}.
    \item[] Justification: Section~\ref{sec:experiments} and Appendix~\ref{appendix-sec:detailed-experimental-setting} specify models, datasets, context lengths, batch sizes, optimizer, selector settings, evaluation cadence, and checkpoint selection. Section~\ref{subsec:main-experiment} describes the tool-calling split and Sudoku reward; Appendix~\ref{appendix-subsec:surface-controls} defines the Stage~2 score controls.
    \end{samepage}

\begin{samepage}
\item {\bf Experiment statistical significance}
    \item[] Question: Does the paper report error bars suitably and correctly defined or other appropriate information about the statistical significance of the experiments?
    \item[] Answer: \answerNo{}.
    \item[] Justification: Due to training cost, the main RL results lack repeated-seed error bars or significance tests. Figure~\ref{fig:entropy_correlation} reports within-bin standard deviations across prompts, not training-run uncertainty. Sudoku and the Stage~2 score controls each use one run per method; avg@4 averages sampled responses, not independent training runs.
    \end{samepage}

\begin{samepage}
\item {\bf Experiments compute resources}
    \item[] Question: For each experiment, does the paper provide sufficient information on the computer resources (type of compute workers, memory, time of execution) needed to reproduce the experiments?
    \item[] Answer: \answerYes{}.
    \item[] Justification: Section~\ref{sec:experiments} specifies 8$\times$A100-80GB GPUs for the main study, and Table~\ref{tab:efficiency_cost} gives its training-time breakdowns.
    \end{samepage}

\begin{samepage}
\item {\bf Code of ethics}
    \item[] Question: Does the research conducted in the paper conform, in every respect, with the NeurIPS Code of Ethics \url{https://neurips.cc/public/EthicsGuidelines}?
    \item[] Answer: \answerYes{}.
    \item[] Justification: The work studies RL post-training efficiency using public math-reasoning and tool-calling resources and a Sudoku task. It does not involve human-subject experiments, private data collection, or deployment decisions.
    \end{samepage}

\begin{samepage}
\item {\bf Broader impacts}
    \item[] Question: Does the paper discuss both potential positive societal impacts and negative societal impacts of the work performed?
    \item[] Answer: \answerYes{}.
    \item[] Justification: Section~\ref{sec:limitations} discusses potential benefits for training cost, energy use, and research accessibility, and acknowledges general risks associated with stronger LLM capabilities. It does not identify application-specific harms beyond those general risks.
    \end{samepage}

\begin{samepage}
\item {\bf Safeguards}
    \item[] Question: Does the paper describe safeguards that have been put in place for responsible release of data or models that have a high risk for misuse (e.g., pre-trained language models, image generators, or scraped datasets)?
    \item[] Answer: \answerNA{}.
    \item[] Justification: The release contains training code and data-preparation utilities, but no new model checkpoints or high-risk datasets requiring controlled release. Experiments use existing public resources and a Sudoku task.
    \end{samepage}

\begin{samepage}
\item {\bf Licenses for existing assets}
    \item[] Question: Are the creators or original owners of assets (e.g., code, data, models), used in the paper, properly credited and are the license and terms of use explicitly mentioned and properly respected?
    \item[] Answer: \answerNo{}.
    \item[] Justification: Table~\ref{tab:assets} records model, data, and software sources, including BFCL-V2, Sudoku, Reasoning Gym, and all six math evaluation sets. Declared upstream licenses are distinguished from missing distribution-level documentation; unresolved data and GRESO licensing details prevent an affirmative answer for all assets. Appendix~\ref{sec:asset-licenses} specifies the Open-R1 30K selection rule and discloses its unrecorded historical revision.
    \end{samepage}

\begin{samepage}
\item {\bf New assets}
    \item[] Question: Are new assets introduced in the paper well documented and is the documentation provided alongside the assets?
    \item[] Answer: \answerYes{}.
    \item[] Justification: The released HIVE code includes a README, installation and data guides, model configurations, runnable commands, attribution, and an Apache-2.0 license. Documentation covers runtime compatibility and dataset-snapshot limitations. The release provides the main math-training pipeline; it does not introduce a new dataset, benchmark, or model checkpoint.
    \end{samepage}

\begin{samepage}
\item {\bf Crowdsourcing and research with human subjects}
    \item[] Question: For crowdsourcing experiments and research with human subjects, does the paper include the full text of instructions given to participants and screenshots, if applicable, as well as details about compensation (if any)?
    \item[] Answer: \answerNA{}.
    \item[] Justification: The paper does not involve crowdsourcing experiments or research with human subjects.
    \end{samepage}

\begin{samepage}
\item {\bf Institutional review board (IRB) approvals or equivalent for research with human subjects}
    \item[] Question: Does the paper describe potential risks incurred by study participants, whether such risks were disclosed to the subjects, and whether Institutional Review Board (IRB) approvals (or an equivalent approval/review based on the requirements of your country or institution) were obtained?
    \item[] Answer: \answerNA{}.
    \item[] Justification: The paper does not involve human subjects, so IRB approval or equivalent review is not applicable.
    \end{samepage}

\begin{samepage}
\item {\bf Declaration of LLM usage}
    \item[] Question: Does the paper describe the usage of LLMs if it is an important, original, or non-standard component of the core methods in this research? Note that if the LLM is used only for writing, editing, or formatting purposes and does \emph{not} impact the core methodology, scientific rigor, or originality of the research, declaration is not required.
    \item[] Answer: \answerYes{}.
    \item[] Justification: The paper's core method is explicitly designed for RL post-training of LLMs/LRMs, and Sections~\ref{sec:method} and~\ref{subsec:main-experiment} describe how the models are used and evaluated.
    \end{samepage}

\end{enumerate}

\end{document}